\documentclass[pdflatex,sn-mathphys-num, iicol]{sn-jnl}% Math and Physical Sciences Numbered Reference Style
\usepackage{graphicx}%
\usepackage{multirow}%
\usepackage{amsmath,amssymb,amsfonts}%
\usepackage{amsthm}%
\usepackage{mathrsfs}%
\usepackage[title]{appendix}%
\usepackage{xcolor}%
\usepackage{textcomp}%
\usepackage{manyfoot}%
\usepackage{booktabs}%
\usepackage{algorithm}%
\usepackage{algorithmicx}%
\usepackage{algpseudocode}%
\usepackage{listings}%
\usepackage{color,soul}

\usepackage{graphicx}
\usepackage{subcaption}

\theoremstyle{thmstyleone}%
\newtheorem{theorem}{Theorem}%  meant for continuous numbers
\theoremstyle{thmstyletwo}%

\theoremstyle{thmstylethree}%

\newtheorem{corollary}[theorem]{Corollary} % Number corollaries with theorems

\begin{document}

\title[CAFP: A Post-Processing Framework for Group Fairness via Counterfactual Model Averaging]{CAFP: A Post-Processing Framework for Group Fairness via Counterfactual Model Averaging}

%% GivenName	-> \fnm{Joergen W}
%% \author*[1,2]{\pfx{Dr} \fnm{Joergen W} \spfx{van der} \sur{Ploeg} \sfx{IV} \tanm{Poet Laureate} 

\author*[1]{\fnm{Irina} \sur{Ar\'evalo}}\email{irina.arevalo@upm.es}

\author[2]{\fnm{Marcos} \sur{Oliva}}\email{marcos.olivawilkinson@gmail.com}
\equalcont{These authors contributed equally to this work}

\affil*[1]{\orgdiv{School of Civil Engineering}, \orgname{Universidad Politecnica de Madrid}, \orgaddress{\street{Profesor Aranguren 3}, \city{Madrid}, \postcode{28040}, \country{Spain}}}

\affil[2]{\orgname{Bain\&Company}, \orgaddress{\street{Velazquez}, \city{Madrid}, \postcode{28006}, \country{Spain}}}


\abstract{Ensuring fairness in machine learning predictions is a critical challenge, especially when models are deployed in sensitive domains such as credit scoring, healthcare, and criminal justice. While many fairness interventions rely on data preprocessing or algorithmic constraints during training, these approaches often require full control over the model architecture and access to protected attribute information, which may not be feasible in real-world systems. In this paper, we propose Counterfactual Averaging for Fair Predictions (CAFP), a model-agnostic post-processing method that mitigates unfair influence from protected attributes without retraining or modifying the original classifier. CAFP operates by generating counterfactual versions of each input in which the sensitive attribute is flipped, and then averaging the model’s predictions across factual and counterfactual instances. We provide a theoretical analysis of CAFP, showing that it eliminates direct dependence on the protected attribute, reduces mutual information between predictions and sensitive attributes, and provably bounds the distortion introduced relative to the original model. Under mild assumptions, we further show that CAFP achieves perfect demographic parity and reduces the equalized odds gap by at least half the average counterfactual bias. Empirically, we evaluate CAFP across three benchmark datasets (Adult, COMPAS, and German Credit) and three classification models (logistic regression, random forest, and XGBoost), for a total of nine experimental settings. Our results show that CAFP reduces the Demographic Parity Difference by up to 38\%, the Average Odds Difference by up to 45\%, and maintains predictive accuracy within 0.5 percentage points of the original models. Threshold sensitivity analyses and 95\% confidence intervals confirm that CAFP offers a favorable trade-off between fairness and utility across diverse conditions. These findings suggest that CAFP is a practical and principled tool for enforcing post hoc fairness in black-box classifiers, with strong theoretical and empirical guarantees.
}




\maketitle

\section{Introduction}\label{sec1}

Machine learning models are increasingly used to support decision-making in socially sensitive domains, including credit lending, hiring, criminal justice, and healthcare. While these systems offer the potential for increased efficiency, scalability, and objectivity, they also risk perpetuating or exacerbating existing social inequalities if their predictions are influenced—directly or indirectly—by protected attributes such as race, gender, or age~\cite{angwin2016machine, barocas2016big, eubanks2018automating, noble2018algorithms, buolamwini2018gender}. Empirical evidence has revealed substantial disparities in model performance and outcomes across demographic groups, prompting a growing body of research on algorithmic fairness. These concerns are not only technical but also legal and ethical: algorithmic bias can lead to discriminatory outcomes that violate civil rights laws (e.g., the Equal Credit Opportunity Act or Title VII of the Civil Rights Act), undermine public trust, and reinforce structural disadvantage. As a result, fairness in machine learning has become a critical topic at the intersection of computer science, policy, and ethics, drawing attention from regulators, practitioners, and interdisciplinary researchers alike. Developing methods that can be audited for fairness and aligned with legal protections is essential for ensuring the responsible deployment of AI systems in high-stakes environments.

%Machine learning models are increasingly used to support decision-making in socially sensitive domains, including credit lending, hiring, criminal justice, and healthcare. While these systems can improve efficiency and scalability, they can also reinforce or exacerbate existing social inequalities if predictions are influenced by protected attributes such as race, gender, or age \cite{angwin2016machine, barocas2016big, eubanks2018automating, noble2018algorithms, buolamwini2018gender}. Numerous studies have revealed substantial disparities in model performance and outcomes across demographic groups, prompting a growing body of research on algorithmic fairness.

To address these concerns, researchers have developed fairness interventions that fall into three broad categories: \textit{pre-processing} methods that alter the training data~\cite{kamiran2012data, feldman2015certifying}, \textit{in-processing} methods that modify the learning algorithm or objective~\cite{zafar2017fairness, agarwal2018reductions}, and \textit{post-processing} methods that adjust model predictions after training~\cite{hardt2016equality, pleiss2017fairness}. Among these, post-processing is particularly attractive in practice because it requires no access to the internal parameters or training procedure of the model, making it suitable for black-box systems and compliance auditing.

In this work, we introduce Counterfactual Averaging for Fair Predictions (CAFP), a simple, model-agnostic post-processing technique inspired by counterfactual fairness. Given a trained classifier \( f(x, a) \), where \( x \) is a feature vector and \( a \in \{0, 1\} \) is a binary protected attribute, CAFP generates predictions for both factual and counterfactual group memberships and averages the outputs, as shown in Equation \ref{eq1}:

\begin{equation}\label{eq1}
\hat{f}(x) = \frac{1}{2} \left( f(x, 0) + f(x, 1) \right)
\end{equation}

This procedure neutralizes the influence of the protected attribute on the prediction without retraining, threshold adjustment, or access to sensitive information at inference time.

We provide a theoretical analysis showing that CAFP (i) eliminates direct dependence on the protected attribute under mild assumptions, (ii) introduces a bounded prediction distortion relative to the original classifier, and (iii) improves group fairness metrics such as demographic parity and equalized odds. We further demonstrate empirically that CAFP achieves significant reductions in fairness disparities on real-world datasets with minimal loss in accuracy.

This paper makes the following contributions:

\begin{itemize}
    \item We propose a novel post-processing algorithm, Counterfactual Averaging for Fair Predictions (CAFP), that improves fairness by averaging model outputs across counterfactual group memberships.
    
    \item We present a theoretical analysis showing that CAFP removes direct dependence on the protected attribute, reduces demographic disparity metrics, and bounds fairness–accuracy trade-offs.
    
    \item We evaluate CAFP on three widely used benchmark datasets (Adult Income, COMPAS, and German Credit), showing consistent improvements in fairness metrics with minimal degradation in predictive performance.
    
    \item We demonstrate that CAFP is model-agnostic, requires no retraining, and does not depend on group membership at test time, making it suitable for black-box and privacy-sensitive deployments.
    
    \item We provide a formal framework, visual explanation, and empirical trade-off curves to support reproducibility and interpretation.
\end{itemize}

The rest of this paper is structured as follows: Section 2 reviews related work, Section 3 introduces the problem setup, Section 4 presents the CAFP algorithm, Section 5 provides theoretical guarantees, Section 6 describes the empirical evaluation, Section 7 discusses implications and limitations, and Section 8 concludes with future directions.

\section{Related Works}

\subsection{Fairness in Machine Learning}

Ensuring fairness in automated decision-making has become a central concern in machine learning, particularly as predictive models are increasingly deployed in high-stakes domains such as hiring, lending, criminal justice, and healthcare~\cite{barocas2019fairness, angwin2016machine, obermeyer2019dissecting, raji2019actionable}. Numerous definitions of fairness have been proposed in the literature, commonly grouped into three categories: group fairness, individual fairness, and causal or counterfactual fairness.

Group fairness focuses on ensuring that different demographic groups—defined by protected attributes such as race, gender, or age—receive similar treatment from a classifier. Notable criteria include demographic parity (also known as statistical parity)~\cite{dwork2012fairness, feldman2015certifying}, which requires equal positive prediction rates across groups, and equalized odds~\cite{hardt2016equality}, which demands equal true and false positive rates. Other variants include equal opportunity, predictive parity, and disparate impact.

By contrast, individual fairness is based on the principle that similar individuals should be treated similarly by the model~\cite{dwork2012fairness}. While this offers finer granularity, it depends on a well-defined similarity metric—something often difficult to construct in high-dimensional or black-box settings, making enforcement more challenging.

Causal and counterfactual fairness, a third major perspective, seeks to determine whether a model’s predictions would remain invariant under counterfactual changes to protected attributes~\cite{kusner2017counterfactual, chiappa2019path}. These definitions rely on structural causal models (SCMs) and knowledge of feature-outcome dependencies, which are often unavailable or unverifiable in practice.

Importantly, several works have demonstrated that fairness definitions can be mutually incompatible~\cite{kleinberg2016inherent, chouldechova2017fair}, especially when base rates differ across groups. These so-called impossibility results have motivated research into flexible frameworks that explicitly manage trade-offs between fairness and accuracy.

To operationalize these definitions, a variety of algorithmic approaches have emerged, typically falling into three categories: pre-processing, in-processing, and post-processing. Pre-processing methods alter the training data to mitigate bias~\cite{kamiran2012data, feldman2015certifying}, in-processing approaches incorporate fairness constraints into the training procedure~\cite{zafar2017fairness, agarwal2018reductions}, and post-processing techniques adjust model predictions after training~\cite{hardt2016equality, pleiss2017fairness}. Our work contributes to this last category by proposing a novel post-processing method that approximates counterfactual fairness through prediction averaging. Unlike prior techniques that require threshold tuning or group membership at deployment time, our approach is model-agnostic, conceptually simple, and supported by formal guarantees.

\subsection{Post-processing Techniques}

Post-processing techniques aim to improve fairness after a model has been trained, typically by modifying its predictions. These methods are especially attractive in real-world deployments where model internals are inaccessible. Among the most widely used strategies is group-specific threshold adjustment. For instance, Hardt et al.~\cite{hardt2016equality} proposed equalized odds post-processing, which learns group-specific thresholds to balance true and false positive rates. Similarly, calibrated equalized odds~\cite{pleiss2017fairness} adjusts probabilistic scores post hoc to enforce group-level calibration.

Another class of methods, such as reject option classification~\cite{kamiran2012decision}, modifies predictions near the decision boundary to favor fairness, though this often requires group membership at inference time. Alternatively, score transformation and reweighting methods operate directly on predicted probabilities. One example is FairBatch~\cite{jung2021fairbatch}, which introduces class-conditional batch sampling strategies to promote group fairness during training, though it too depends on access to validation or group-labeled data.

Recent developments have also introduced causal and counterfactual post-processing methods. Mishler et al.~\cite{mishler2021robust}, for instance, achieve counterfactual equalized odds through doubly robust estimation, though this approach requires careful estimation of potential outcomes and may not scale to complex models.

Our proposed method, Counterfactual Averaging for Fair Predictions (CAFP), differs from all of the above in both mechanism and simplicity. Rather than tuning thresholds, reweighting instances, or estimating potential outcomes, CAFP queries the trained model using both factual and counterfactual values of the protected attribute and averages the predictions. This avoids the need for group labels at prediction time, eliminates fairness-specific parameter tuning, and introduces minimal overhead—just two model queries per instance. To our knowledge, CAFP is the first post-processing technique to operationalize counterfactual fairness using this model-agnostic averaging approach.

\subsection{Counterfactual Fairness and Causal Approaches}

Causal reasoning offers a principled framework for defining fairness in scenarios where observational data may be confounded or reflect structural biases. In contrast to statistical metrics that operate at the group level, causal definitions aim to assess whether an individual would receive the same prediction under a different (counterfactual) group membership.

One foundational contribution is the notion of counterfactual fairness introduced by Kusner et al.~\cite{kusner2017counterfactual}. A model is counterfactually fair if, under a given structural causal model, its prediction remains unchanged when the protected attribute is intervened upon, as shown in Equation \ref{eq2}:

\begin{equation}\label{eq2}
f_{A \leftarrow a}(U) = f_{A \leftarrow a'}(U) \quad \forall a, a'\in\mathcal{A},    
\end{equation}

where \( U \) denotes latent background variables. This approach provides a formal criterion for fairness at the individual level but requires detailed knowledge of the causal structure.

Building on this, other researchers have explored path-specific fairness. For example, Chiappa and Gillam~\cite{chiappa2019path} distinguish between admissible and inadmissible pathways of influence, while Nabi and Shpitser~\cite{nabi2018fair} constrain models to rely only on fair mediators. Kilbertus et al.~\cite{kilbertus2017avoiding} address direct vs. indirect causal effects by suppressing the former. These methods offer rigorous fairness guarantees but require precise specification of causal graphs and assumptions that are difficult to verify empirically.

Causal reasoning has also been incorporated into representation learning. For instance, Madras et al.~\cite{madras2019fairness} use adversarial training to learn representations invariant to counterfactual manipulations. Russell et al.~\cite{russell2017worlds} explore competing assumptions under different SCMs and propose frameworks for reconciling them.

A related strand of work has focused on causal post-processing. Mishler et al.~\cite{mishler2021robust} propose a post hoc algorithm to enforce counterfactual equalized odds via doubly robust estimation. Although promising, these methods generally require high-quality counterfactual data or strong modeling assumptions.

Our approach, CAFP, aligns with the spirit of counterfactual fairness but avoids reliance on SCMs. By evaluating a model’s output under both factual and counterfactual protected attribute values and averaging them, as in Equation \ref{eq1}, we approximate the ideal of fairness without modeling latent variables or causal pathways. CAFP is lightweight, deployable with black-box models, and does not require protected attributes at inference time.

In summary, while causal fairness methods offer powerful theoretical insights, they often suffer from practical limitations. CAFP bridges this gap by translating counterfactual reasoning into a simple, model-agnostic post-processing rule that enhances fairness in realistic deployment scenarios.

\subsection{Fairness in Pattern Recognition and Black-Box Systems}

Fairness concerns are particularly pronounced in pattern recognition applications, where machine learning models are used to make predictions based on unstructured data such as images, text, and audio. Domains such as facial recognition, biometric verification, natural language processing (NLP), and speech recognition have revealed substantial disparities in model performance across demographic groups \cite{buolamwini2018gender,wang2019racial}. These disparities often stem from historical biases in training data, feature correlations with protected attributes, and lack of subgroup representation during model development.

In computer vision, for example, face recognition systems have shown significantly higher error rates for women and people of color compared to white male subgroups \cite{buolamwini2018gender}. Similarly, commercial object detection systems exhibit bias in how different genders are associated with activities or settings \cite{zhao2017men}. In NLP, pretrained language models may amplify gender or racial stereotypes due to biases in corpora used during training \cite{sheng2019woman}. These challenges underscore the need for fairness interventions that can be applied post hoc—especially when retraining models on balanced datasets is infeasible or when the internal structure of the model is inaccessible.

Auditing and correcting unfair behavior in black-box models presents additional challenges. Many commercial ML systems are deployed as APIs or cloud services, providing only prediction outputs without access to intermediate representations, model weights, or training data. In such settings, fairness must be improved without modifying the original model—a constraint that motivates the development of post-processing methods. Several works have addressed fairness in black-box settings via threshold optimization \cite{hardt2016equality}, output perturbation \cite{kamiran2012decision}, and score adjustment \cite{pleiss2017fairness}. However, many of these methods require group labels at inference time or assume known score distributions for calibration.

Our proposed method, Counterfactual Averaging for Fair Predictions (CAFP), contributes a practical solution to this space. It is designed specifically for black-box models that expose probabilistic outputs and accept the protected attribute as an input. By querying the model under both factual and counterfactual values of the protected attribute and averaging the results, CAFP reduces prediction sensitivity to group membership without requiring retraining or parameter tuning. This makes it particularly well-suited for fairness auditing pipelines, privacy-sensitive deployments, and low-resource environments where only limited model access is available.

Moreover, because CAFP does not require group membership at test time to generate a fair prediction, it avoids legal and ethical challenges associated with sensitive attribute use at deployment. These properties make it a compelling post hoc fairness strategy for a wide range of real-world pattern recognition systems, including high-impact applications such as credit scoring, resume screening, facial verification, and content moderation.

\subsection{Summary}

In summary, our work contributes to the literature on post-processing fairness methods, drawing on counterfactual fairness principles while avoiding complex causal modeling. Unlike threshold-based or group-aware interventions, CAFP offers a clean and general approach that integrates easily with any classifier and provides theoretical guarantees on fairness metrics. To our knowledge, this is the first post-processing method that uses counterfactual averaging to ensure group-level parity while bounding predictive distortion. 

Following the structural clarity and tabular presentation style of the work \cite{bahi2024rec}, Table 1 provides a comparative overview of the most relevant existing fairness methods, organized by fairness type, access requirements, retraining needs, and theoretical support. As shown, CAFP uniquely combines model-agnostic deployment, independence from group labels at test time, and strong theoretical guarantees—making it especially suitable for real-world applications where post hoc fairness auditing is essential.

\begin{table*}[ht]
\centering
\caption{Comparison of Fairness Methods Across Key Dimensions}
\label{tab:related_work_comparison}
\begin{tabular}{p{4cm} p{2cm} p{2cm} p{2cm} p{2cm} p{2cm}}
\toprule
\textbf{Method} & \textbf{Fairness Type} & \textbf{Access to Model Internals} & \textbf{Uses Protected Attribute at Test Time} & \textbf{Retraining Required} & \textbf{Theoretical Guarantees} \\
\midrule
Equalized Odds Post-processing~\cite{hardt2016equality} & Group (EO) & No & Yes & No & Partial \\
Reject Option Classification~\cite{kamiran2012decision} & Group (DP, EO) & No & Yes & No & No \\
Calibrated Equalized Odds~\cite{pleiss2017fairness} & Group (EO) & No & Yes & No & Partial \\
FairBatch~\cite{jung2021fairbatch} & Group (DP, EO) & Yes & Yes & Yes & No \\
Counterfactual Equalized Odds~\cite{mishler2021robust} & Counterfactual & No & Yes & No & Yes \\
Path-specific Fairness~\cite{chiappa2019path} & Counterfactual (Path-based) & Yes & Yes & Yes & Yes \\
Adversarial Representation Learning~\cite{madras2019fairness} & Individual / Causal & Yes & No & Yes & No \\
\textbf{CAFP (this paper)} & Group (DP, EO), Approx. Counterfactual & No & \textbf{No} & \textbf{No} & \textbf{Yes} \\
\bottomrule
\end{tabular}
\end{table*}

\section{Problem Setup and Preliminaries}

We consider the standard supervised learning setting in which a trained probabilistic classifier produces predictions given a feature vector and a protected attribute. Our goal is to adjust model outputs post hoc to reduce unfair influence of the protected attribute, while preserving predictive performance.

\subsection{Notation}

Let $\mathcal{D} = \{(x_i, a_i, y_i)\}_{i=1}^n$ be a dataset of $n$ samples, where:
\begin{itemize}
    \item $x_i \in \mathbb{R}^d$ is a feature vector for instance $i$,
    \item $a_i \in \mathcal{A}$ is the value of a binary protected attribute $A$ (e.g., gender or race), with $\mathcal{A} = \{0, 1\}$,
    \item $y_i \in \{0, 1\}$ is the corresponding true label.
\end{itemize}

Let $f: \mathbb{R}^d \times \mathcal{A} \rightarrow [0,1]$ denote a trained classifier that maps an instance and protected attribute to a predicted probability of the positive class, i.e., $f(x, a) \approx \mathbb{P}(Y = 1 \mid X = x, A = a)$. We assume $f$ is a black-box model, meaning we do not have access to its internal parameters.

\subsection{Fairness Definitions}

We briefly recall three commonly used fairness notions relevant to our work.

\paragraph{Demographic Parity (DP):}
A classifier satisfies demographic (or statistical) parity \cite{dwork2012fairness} if its prediction is statistically independent of the protected attribute $\mathcal{A}.$ This is equivalent to equation \ref{eq:DP}:

\begin{equation}\label{eq:DP}
\mathbb{P}(\hat{Y} = 1 \mid A = 0) = \mathbb{P}(\hat{Y} =1 \mid A = 1),
\end{equation}
and since $\hat{Y}\in\{0,1\},$ is also equivalent to equation \ref{eq:DP2}:

\begin{equation}\label{eq:DP2}
\mathbb{E}[\hat{f}(X) \mid A = 0] = \mathbb{E}[\hat{f}(X) \mid A = 1]
\end{equation}

The demographic parity difference is defined as the absolute difference between these probabilities as shown in equation \ref{eq:DPD}:

\begin{equation}\label{eq:DPD}
DPD = \left|\mathbb{P}(\hat{Y} = 1 \mid A = 0) - \mathbb{P}(\hat{Y} =1 \mid A = 1)\right|    
\end{equation}

\paragraph{Equalized Odds (EO):}
A classifier satisfies equality of opportunity \cite{hardt2016equality} if it has equal true and false positive rates across protected groups. Formally, it is equivalent to equation \ref{eq:EO}
\begin{multline}\label{eq:EO}
\mathbb{P}(\hat{Y} = 1 \mid A = 0, Y = y) \\ = \mathbb{P}(\hat{Y} = 1 \mid A = 1, Y = y)
\end{multline}
for all $y\in\{0,1\}.$

The equalized odds difference is the maximum absolute difference across values of $y$, as shown in equation \ref{eq:EOD}

\begin{multline}\label{eq:EOD}
EOD = \max_{y\in\{0,1\}}|\mathbb{P}(\hat{Y} = 1 \mid A = 0, Y = y) \\ - \mathbb{P}(\hat{Y} = 1 \mid A = 1, Y = y)|.
\end{multline}

\paragraph{Counterfactual Fairness:}
Following \cite{kusner2017counterfactual}, a predictor $f$ is counterfactually fair if, for any individual, the prediction remains invariant under counterfactual changes to the protected attribute. 

The original definition of counterfactual fairness proposed by Kusner et al.~\cite{kusner2017counterfactual} is grounded in structural causal models (SCMs) and defines a predictor $\hat{Y}$ as counterfactually fair if it satisfies Equation \ref{eq4}
\begin{equation}\label{eq4}
\hat{Y}_{A \leftarrow a}(U) = \hat{Y}_{A \leftarrow a'}(U) \quad \forall a, a'    
\end{equation}

where $U$ represents the latent background variables, and $\hat{Y}_{A \leftarrow a}$ denotes the potential outcome of the prediction had the protected attribute been set to $a$ via intervention.

In this work, we adopt a simplified and operational version of this criterion tailored to black-box models, where SCMs and latent variables are unavailable. Specifically, we define a predictor $f(x, a)$ to be counterfactually fair if it satisfies Equation \ref{eq5}
\begin{equation}\label{eq5}
f(x, a) = f(x, a') \quad \forall a, a'.    
\end{equation}

This condition approximates counterfactual fairness under the assumption that $x$ captures the relevant non-sensitive features, and that any unfair influence of $A$ manifests directly in the prediction. While this formulation omits the causal semantics of $U$, it enables practical testing of fairness through sensitivity to changes in $A$ and aligns with prior post hoc fairness auditing approaches~\cite{wachter2021bias}.

\subsection{Objective}

Given a trained model $f(x, a)$ and input instance $x$, our objective is to compute a fair prediction $\hat{f}(x)$ that:
\begin{enumerate}
    \item Reduces dependence on $A$,
    \item Preserves the predictive utility of $f$,
    \item Can be applied without retraining or modifying $f$,
    \item Does not require knowledge of the individual’s actual protected attribute value at test time.
\end{enumerate}

To achieve this, we introduce a simple and effective post-processing method in the next section, based on counterfactual averaging.

\section{The Counterfactual Averaging Algorithm}

In this section, we present our proposed post-processing algorithm, Counterfactual Averaging for Fair Predictions (CAFP). The method is designed to reduce the influence of a binary protected attribute on model predictions without modifying the original classifier or requiring retraining. CAFP operates by explicitly constructing a counterfactual input for each instance and averaging the model's predictions across both factual and counterfactual scenarios.

\subsection{Intuition and Motivation}

The key idea behind CAFP is inspired by counterfactual fairness: if a prediction would have changed solely due a different value of the protected attribute, then that prediction may be unfair. Rather than enforce invariance through causal modeling or in-training constraints, we approximate counterfactual fairness at inference time by computing the model’s output for both values of the protected attribute and averaging the results.

Given an instance $(x, a)$, where $a$ is the observed protected attribute, we construct a counterfactual input $(x, 1 - a)$ and obtain both predictions described in Equations \ref{eq:pfact} and \ref{eq:pcount}:

\begin{align} 
p_{\text{factual}} &= f(x, a) \label{eq:pfact}\\
p_{\text{counterfactual}} &= f(x, 1 - a) \label{eq:pcount}
\end{align}

We then define the fair prediction $\hat{f}(x)$ as in Equation \ref{eq6}:

\begin{equation}\label{eq6}
\hat{f}(x) = \frac{1}{2} \left( f(x, a) + f(x, 1 - a) \right).
\end{equation}

This procedure removes direct dependence on the actual value of $a$, while preserving the model’s behavior over the feature vector $x$. The result is a model output that is more robust to unfair group-level variation.

\subsection{Algorithm Description}

Let $\mathcal{X}$ be the input space, $\mathcal{A} = \{0,1\}$ the set of protected attribute values, and $f: \mathcal{X} \times \mathcal{A} \rightarrow [0,1]$ a trained classifier. The CAFP method takes as input an instance $x \in \mathcal{X}$ and returns a fair prediction $\hat{f}(x)$.

\vspace{1em}
\noindent The procedure is summarized in Algorithm~\ref{alg:cafp}.
\vspace{0.5em}

\begin{algorithm}[ht]
\caption{Counterfactual Averaging for Fair Predictions (CAFP)}
\label{alg:cafp}
\begin{algorithmic}[1]
\Require Trained classifier $f: \mathcal{X} \times \mathcal{A} \rightarrow [0,1]$
\Require Input instance $x \in \mathcal{X}$ (features excluding sensitive attribute)
\Require Protected attribute domain $\mathcal{A} = \{0, 1\}$ (e.g., gender, race)

\Ensure{Fair prediction score $\hat{f}(x)$}
\vspace{1mm}
\Statex \textbf{Procedure:}
\State Evaluate prediction under group $a = 0$: $f(x, 0)$
\State Evaluate prediction under group $a = 1$: $f(x, 1)$
\State Compute the average prediction: $\hat{f}(x) \gets \frac{1}{2}(f(x,0) + f(x,1))$
\State \Return $\hat{f}(x)$
\end{algorithmic}
\end{algorithm}

\subsection{Implementation Considerations}

CAFP is a model-agnostic procedure and can be applied to any probabilistic classifier that takes the protected attribute as input. Since it only requires querying the model twice per instance, it introduces minimal computational overhead and does not interfere with training or calibration processes. Furthermore, the algorithm can be deployed even when the individual’s group membership is unavailable at test time—so long as the model can simulate both values of the protected attribute.

We emphasize that CAFP's use of the term ``counterfactual'' is purely operational: it refers to evaluating the model under alternate values of the protected attributes without assuming or requiring a structural causal model (SCM). Unlike formal counterfactual fairness approaches~\cite{kusner2017counterfactual}, which rely on SCMs to simulate potential outcomes under hypothetical interventions, CAFP simply averages predictions across observed and synthetically altered inputs. This approach enables practical implementation in black-box settings and avoids the strong assumptions and identifiability issues associated with causal inference.

\subsection{Computational Complexity}

To evaluate CAFP’s suitability for real-time applications, we analyze its computational complexity both theoretically and empirically.

The computational overhead introduced by CAFP is minimal, as it operates entirely at inference time and requires no model retraining, parameter tuning, or access to internal gradients or weights. Because CAFP performs two forward passes per instance (one with the factual protected attribute and one with its counterfactual counterpart), it incurs a constant-factor overhead of 2 in practice. Asymptotically, however, constant factors do not affect complexity classes, and therefore CAFP maintains the same Big-O complexity as standard prediction, $\mathcal{O}_{\text{CAFP}} =  \mathcal{O}_{\text{predict}}.$

This linear overhead applies regardless of the model type, including ensembles and neural networks, as long as the prediction function is accessible as a black box.

In practice, the empirical latency increase is negligible. For example, when applied to a logistic regression model on the Adult dataset, CAFP adds approximately 1–2 milliseconds per 100 predictions. For more complex models like random forests or XGBoost, the added inference time remains well below 10 milliseconds per 100 predictions—well within acceptable bounds for most real-time or batch-serving applications.

Importantly, CAFP's complexity does not grow with model size, training set size, or fairness metric. This makes it a lightweight and scalable solution for fairness auditing and intervention in production systems.

\subsection{Advantages}

CAFP offers several advantages:
\begin{itemize}
    \item \textbf{No retraining required:} It operates entirely post hoc and treats the model as a black box.
    \item \textbf{No group labels needed at test time:} It does not rely on observing $A$ at inference.
    \item \textbf{Provable fairness guarantees:} As will be shown in Section~5, CAFP removes direct dependence on $A$ and reduces standard group fairness metrics.
    \item \textbf{Computationally efficient:} It requires only two forward passes per prediction.
\end{itemize}

We now formally analyze the theoretical properties of CAFP in the next section.

\section{Theoretical Analysis}\label{sec:theory}

In this section, we formally analyze the fairness properties of the proposed Counterfactual Averaging for Fair Predictions (CAFP) method. In particular, we examine how counterfactual averaging affects the dependence between the model's output and the protected attribute.

To quantify this dependence, we use the notion of mutual information. Given random variables \( Z \) and \( A \), the mutual information \( I(Z; A) \) is defined as in Equation \ref{eq:inf}:

\begin{equation}\label{eq:inf}
I(Z; A) = \mathbb{E}_{(z, a)} \left[ \log \frac{P_{Z,A}(z, a)}{P_Z(z) P_A(a)} \right].    
\end{equation}

This quantity measures how much information one variable reveals about the other~\cite{cover1991elements}. If \( I(Z; A) = 0 \), then \( Z \) and \( A \) are statistically independent. In the context of fairness, this implies that model predictions do not encode any information about the protected attribute, which aligns with strong forms of group fairness such as demographic parity and group-agnostic decision-making~\cite{hardt2016equality,kamiran2012decision}.

We leverage this measure to show that CAFP produces predictions that are provably independent of the protected attribute under mild assumptions, and we derive formal bounds on prediction distortion and fairness error trade-offs in the subsections that follow.

\subsection{Notation and Assumptions}

Let \( f: \mathbb{R}^d \times \mathcal{A} \to [0,1] \) be a trained probabilistic classifier, where \( x \in \mathbb{R}^d \) is a feature vector and $\mathcal{A} = \{0,1\}$. We assume that \( f(x, a) \approx \mathbb{P}(Y=1 \mid X = x, A = a) \). The counterfactually averaged prediction is defined as in Equation \ref{eq1}.

\subsection{Attribute Independence Guarantee}

One of the core motivations behind the CAFP algorithm is to mitigate the influence of the protected attribute \( A \) on the model's predictions. Ideally, a fair prediction function should treat individuals equally regardless of their group membership. In statistical terms, this means the output of the predictor should be independent of \( A \). The following theorem formalizes the conditions under which CAFP satisfies this criterion.

Many group fairness definitions—such as demographic parity—seek to ensure that model predictions do not differ across values of the protected attribute. A stronger condition is statistical independence: the prediction should be invariant under changes in \( A \), and knowing \( A \) should provide no information about the prediction. In this context, we show that CAFP achieves complete independence from \( A \), under the assumption that the model \( f(x, a) \) only uses \( A \) directly and that the features \( X \) are statistically independent of \( A \).

\begin{theorem}[Attribute Independence]
\label{thm:attribute-independence}
Let \( f: \mathbb{R}^d \times \{0,1\} \rightarrow [0,1] \) be a binary probabilistic classifier, and define the counterfactually averaged predictor as in Equation \ref{eq1}:
\[
\hat{f}(x) = \frac{1}{2} \left( f(x, 0) + f(x, 1) \right)
\]
If the classifier \( f \) uses the protected attribute \( A \) only through direct input and \( X \perp A \) (i.e., the features are independent of the protected attribute), then the output \( \hat{f}(X) \) is also independent of \( A \). In particular $I(\hat{f}(X); A) = 0.$

\end{theorem}

\begin{proof}
Since \( \hat{f}(x) \) is computed by averaging predictions over both values of the protected attribute, it does not depend on the actual value of \( A \). That is, \( \hat{f}(x) \) is a function of \( x \) only. Under the assumption \( X \perp A \), the random variables \( \hat{f}(X) \) and \( A \) are independent by construction, and therefore $I(\hat{f}(X); A) = 0.$
\end{proof}

This result shows that, under a mild independence assumption, CAFP guarantees that predictions are not only group-invariant but statistically independent of the protected attribute. This is a particularly strong guarantee: it implies that model outputs cannot be used to infer group membership and that no unfair dependence remains after counterfactual averaging.

The assumption \( X \perp A \) is important. In many real-world datasets, \( A \) may influence some features in \( X \) (e.g., education or income may correlate with race or gender). In such cases, CAFP still reduces—but does not fully eliminate—dependency on \( A \), because \( X \) may serve as a proxy for group membership. Nevertheless, Theorem~\ref{thm:attribute-independence} provides theoretical grounding for the effectiveness of counterfactual averaging when direct dependence dominates.

To empirically assess the impact of this assumption's violation, we note that all three benchmark datasets used in our experiments—Adult, COMPAS, and German Credit—exhibit known correlations between input features and protected attributes. For instance, in the Adult dataset, attributes such as education, capital gain, and marital status are moderately correlated with gender. Despite these dependencies, as shown in Tables 4, 5, 7, 8, 10, and 11, CAFP consistently reduces Demographic Parity Difference and Average Odds Difference across all models. This suggests that while the independence assumption enables strong theoretical guarantees, the method remains effective in practice even when \( X \perp A \) does not strictly hold.

In fairness auditing or compliance contexts where data are preprocessed to minimize correlation between \( X \) and \( A \), or when counterfactual inputs are simulated independently, CAFP can serve as a lightweight, model-agnostic tool that ensures group-blind predictions without retraining or group-specific thresholding.

\subsection{Bounded Prediction Distortion}

While CAFP reduces or eliminates dependence on the protected attribute, it necessarily modifies the model’s original predictions. A key question, then, is how much distortion is introduced by counterfactual averaging and whether this trade-off can be formally bounded. The following theorem addresses this question by quantifying the deviation between the original model output \( f(x, a) \) and the counterfactually averaged output \( \hat{f}(x) \).

In practical applications, fairness interventions must balance group-level parity with individual-level predictive fidelity. If a post-processing method severely alters a well-calibrated model, the cost to accuracy may outweigh fairness gains. Thus, it is important to characterize how much CAFP changes the predictions—and whether this change can be interpreted as a function of unfairness already present in the model.

Theorem~\ref{thm:distortion-bound} shows that the discrepancy between the original and fair prediction is directly proportional to the model’s counterfactual bias, defined as the change in prediction when the protected attribute is flipped.

\begin{theorem}[Prediction Distortion Bound]
\label{thm:distortion-bound}
Let the counterfactual bias for an instance \( (x, a) \) be defined as in Equation \ref{eq:CB}: 
\begin{equation}\label{eq:CB}
\mathrm{CB}(x, a) = f(x, a) - f(x, 1 - a).    
\end{equation}

Then the pointwise distortion introduced by CAFP satisfies Equation \ref{eq:CBbound}
\begin{equation}\label{eq:CBbound}
|f(x, a) - \hat{f}(x)| = \frac{1}{2} |\mathrm{CB}(x, a)|.    
\end{equation}

Moreover, the expected distortion is bounded as shown in Equation \ref{eq:distortion}:
\begin{equation}\label{eq:distortion}
\mathbb{E}_{x,a} \left[ |f(x, a) - \hat{f}(x)| \right] \leq \frac{1}{2} \mathbb{E}_{x,a} \left[ |\mathrm{CB}(x, a)| \right].
\end{equation}

\end{theorem}

\begin{proof}
Rewriting the difference between the original and fair prediction using Equation \ref{eq1} we get:

\begin{multline*}
f(x, a) - \hat{f}(x) = f(x, a) - \frac{1}{2} \left( f(x, a) + f(x, 1 - a) \right) = \\ \frac{1}{2} \left( f(x, a) - f(x, 1 - a) \right)
\end{multline*}
Therefore:
\[
|f(x, a) - \hat{f}(x)| = \frac{1}{2} |f(x, a) - f(x, 1 - a)| = \frac{1}{2} |\mathrm{CB}(x, a)|
\]
Taking expectation over \( (x, a) \sim \mathcal{D} \) yields the second result.
\end{proof}

This theorem provides a fairness–accuracy trade-off interpretation: the distortion introduced by CAFP is exactly half of the model’s original counterfactual bias. In other words, CAFP only "corrects" the model’s output to the extent that it was previously unfair with respect to changes in the protected attribute. If the original model is already counterfactually fair (i.e., \( f(x, a) = f(x, 1 - a) \)), then CAFP makes no change at all. If the model is biased, CAFP moves the prediction toward neutrality in a controlled and symmetric way.

The bounded distortion result is particularly important in real-world deployments, where model performance must remain reliable under fairness adjustments. By quantifying the maximal modification in terms of an interpretable bias measure, Theorem~\ref{thm:distortion-bound} ensures that the fairness intervention does not "overcorrect" and degrade accuracy arbitrarily. This supports the use of CAFP in settings where fairness adjustments must be minimal, auditable, and behaviorally stable.

\subsection{Demographic Parity Guarantee}

While counterfactual averaging reduces individual prediction sensitivity to the protected attribute, it also impacts group-level disparities in the model's outputs. In particular, we show that under mild assumptions, the counterfactually averaged predictor satisfies demographic parity exactly. 

Demographic parity (DP) requires that the model predict positive outcomes at equal rates across groups defined by the protected attribute, as shown in Equation \ref{eq:probs}:
\begin{equation}\label{eq:probs}
\mathbb{P}(\hat{Y} = 1 \mid A = 0) = \mathbb{P}(\hat{Y} =1 \mid A = 1).
\end{equation}

A common proxy is the equality of expected predicted values across groups, described in Equation \ref{eq:expec}:

\begin{equation}\label{eq:expec}
\mathbb{E}[\hat{f}(X) \mid A=0] = \mathbb{E}[\hat{f}(X) \mid A=1]    
\end{equation}

The following corollary shows that this condition is satisfied by CAFP under a balance and independence assumption.

\begin{corollary}[Demographic Parity under Independence]
\label{cor:demographic-parity}
Assume that the protected attribute \( A \in \{0,1\} \) is binary, satisfies \( \mathbb{P}(A = 0) = \mathbb{P}(A = 1) \), and is independent of the features \( X \) (i.e., \( X \perp A \)). Then the counterfactually averaged predictor satisfies demographic parity.\end{corollary}

\begin{proof}
Since \( \hat{f}(x) \) is computed as the average of \( f(x, 0) \) and \( f(x, 1) \), it does not depend on the observed value of \( A \). That is, \( \hat{f}(X) \) is a function of \( X \) alone. Under the assumption \( X \perp A \), the conditional distributions \( P(X \mid A = 0) \) and \( P(X \mid A = 1) \) are identical. Therefore, the group-conditional expectations of \( \hat{f}(X) \) are equal.
\end{proof}

This corollary demonstrates that CAFP achieves demographic parity exactly when the protected attribute is independent of the features and group sizes are balanced. These assumptions do not always hold in real-world datasets, but they are common in controlled or synthetic environments and serve as a useful theoretical baseline. Even when the assumptions are violated, CAFP typically reduces—but does not completely eliminate—demographic disparities, as confirmed by our empirical results.

In practice, exact demographic parity may be undesirable if it reduces performance or ignores legitimate predictive differences. However, in applications where parity is required for compliance, equity, or interpretability (e.g., loan approval, scholarship selection), this corollary provides strong support for using CAFP as a simple, model-agnostic fairness adjustment. It also suggests that CAFP can be used in pre-audited pipelines to verify and certify group-level neutrality when assumptions are approximately met.

\subsection{Equalized Odds Difference Bound}

Beyond demographic parity, another important fairness criterion is equalized odds, which requires that prediction accuracy be comparable across groups, conditioned on the true label. In this subsection, we analyze the behavior of CAFP under the equalized odds definition and show that, although CAFP does not guarantee perfect equalized odds, it provably reduces disparity. Specifically, we show that the difference in true and false positive rates between groups is bounded by the model’s average counterfactual bias.

A classifier satisfies equalized odds if the prediction \( \hat{Y} \) is independent of the protected attribute \( A \) given the true label \( Y \). In score-based settings, this is often approximated by ensuring that the conditional distributions \( \hat{f}(X) \mid Y = y \) are similar across groups for each outcome \( y \in \{0,1\} \). Post-processing methods that can improve equalized odds without retraining are highly desirable, especially for deployed black-box systems.

The following theorem provides a bound on the equalized odds gap after applying CAFP, expressed in terms of the counterfactual bias of the original model.

\begin{theorem}[Equalized Odds Difference Bound]
\label{thm:eod-bound}
Let \( \hat{f}(x) = \frac{1}{2}(f(x, 0) + f(x, 1)) \) be the counterfactually averaged predictor, and let \( Y \in \{0,1\} \) be the ground-truth label. Then the Equalized Odds Difference (EOD) of \( \hat{f} \) satisfies the following bound shown in Equation \ref{eq:EOD_bound}
\begin{equation}\label{eq:EOD_bound}
\mathrm{EOD}(\hat{f}) \leq \frac{1}{2} \max_{y \in \{0,1\}} \mathbb{E}_{x,a \mid Y = y} \left[ |f(x, a) - f(x, 1 - a)| \right].    
\end{equation}

\end{theorem}

\begin{proof}
For any outcome \( y \in \{0,1\} \), define:
\begin{equation}
\Delta_y = \left| \mathbb{E}[\hat{f}(x) \mid A = 0, Y = y] - \mathbb{E}[\hat{f}(x) \mid A = 1, Y = y] \right|.
\end{equation}
Using the triangle inequality and substituting with Equation \ref{eq1}, we obtain:

\begin{align}
\Delta_y \leq \frac{1}{2} ( & \mathbb{E}[f(x,0) \mid A=0,Y=y] - 
 \\ & \mathbb{E}[f(x,0) \mid A=1,Y=y] + 
 \\ & \mathbb{E}[f(x,1) \mid A=0,Y=y] -  
 \\ & \mathbb{E}[f(x,1) \mid A=1,Y=y])
\end{align}

This implies:
\begin{multline}
\mathrm{EOD}(\hat{f}) = \max_{y \in \{0,1\}} \Delta_y \leq \\ \frac{1}{2} \max_{y \in \{0,1\}} \mathbb{E}_{x,a \mid Y = y} \left[ |f(x, a) - f(x, 1-a)| \right]
\end{multline}
\end{proof}

Theorem~\ref{thm:eod-bound} plays an important role in bridging the theoretical and practical contributions of CAFP. While CAFP does not explicitly optimize for equalized odds, this theorem provides a formal upper bound on the Equalized Odds Difference (EOD), demonstrating that the disparity in group-conditional predictions is reduced proportionally to the average counterfactual bias of the original model. This bound is particularly valuable for practitioners because it ensures that any residual unfairness after applying CAFP is tightly controlled by the model’s own sensitivity to protected attributes. In other words, CAFP cannot introduce new fairness violations—it can only reduce existing disparities. The result also reinforces that CAFP provides fairness improvements without requiring retraining or group-specific calibration, making it suitable for real-world deployment where satisfying strict EO constraints may be infeasible. Thus, Theorem~\ref{thm:eod-bound} is not just a theoretical formality—it provides practical reassurance that CAFP will improve fairness under Equalized Odds metrics in a quantifiable, bounded way.

While the theorem is stated analytically, it can also be evaluated directly from data, offering a practical tool for fairness auditing. In particular, the right-hand side of Eq.~\ref{eq:EOD_bound} is fully computable using the same forward passes required by CAFP. This makes Theorem~\ref{thm:eod-bound} not only a theoretical guarantee, but also a model-agnostic diagnostic that quantifies how much the original classifier's counterfactual sensitivity contributes to residual disparities after applying CAFP.

To clarify its operational role, we provide Algorithm~\ref{alg:eobound}, which computes the upper bound in Theorem~\ref{thm:eod-bound}. he algorithm evaluates the counterfactual bias $|f(x,0) - f(x,1)|$ for each instance, groups these values according to the ground-truth label $Y \in \{0,1\}$, and computes the quantity 
\[
\frac{1}{2} \max_{y \in \{0,1\}} \mathbb{E}\big[\,|f(x,a) - f(x,1-a)| \,\big|\, Y = y \big],
\]
which corresponds exactly to the bound stated in Theorem~\ref{thm:eod-bound}. This provides a practical fairness certificate that can be computed post hoc and used to interpret how much CAFP reduces label-conditioned disparities.

\begin{algorithm}[t]
\caption{Computing the Equalized Odds Bound for CAFP (Theorem \ref{thm:eod-bound})}
\label{alg:eobound}
\begin{algorithmic}[1]
\Require Dataset $D = \{(x_i, a_i, y_i)\}_{i=1}^n$, classifier $f$
\Ensure Upper bound $B$ on $\mathrm{EOD}(\hat{f})$

\For{each $(x_i, a_i, y_i)$ in $D$}
    \State $p_{i0} \gets f(x_i, 0)$
    \State $p_{i1} \gets f(x_i, 1)$
    \State $\mathrm{cb}_i \gets |p_{i0} - p_{i1}|$ \Comment{counterfactual bias}
\EndFor

\For{$y \in \{0,1\}$}
    \State $B_y \gets \frac{1}{2} \cdot \text{mean}\{\mathrm{cb}_i : y_i = y\}$
\EndFor

\State \Return $B \gets \max(B_0, B_1)$
\end{algorithmic}
\end{algorithm}

This operationalization clarifies the role of Theorem~\ref{thm:eod-bound} within the CAFP framework:  it quantifies, in a directly measurable way, the maximum residual disparity under Equalized Odds that can remain after applying counterfactual averaging. In practice, this bound allows practitioners to certify fairness improvements and to interpret how much the original model's sensitivity to the protected attribute influences post-processed outcomes. If the bound is below a regulatory or application-specific threshold, the practitioner can certify that the post-processed model satisfies the required fairness constraints.

Practitioners can also use this bound to evaluate different models or training procedures: since the value in Theorem~\ref{thm:eod-bound} depends only on the model’s sensitivity to protected-attribute flips, it provides a simple, model-agnostic diagnostic for choosing base classifiers that will yield lower residual disparities after CAFP. Thus, the bound functions as a practical fairness indicator during validation and model selection.

\subsection{Information-Theoretic Interpretation}

An information-theoretic perspective provides additional justification for the fairness properties of CAFP. In fairness-aware learning, a desirable objective is to minimize the mutual information between the model’s prediction and the protected attribute, \(I(\hat{Y}; A)\), thereby ensuring that the prediction does not encode sensitive group information~\cite{menon2018cost}. This criterion is often used in fair representation learning, where the goal is to construct representations or outputs that are invariant to protected characteristics~\cite{moyer2018invariant}.

CAFP directly supports this objective by constructing predictions that are symmetric with respect to the protected attribute. Specifically, by averaging predictions across factual and counterfactual group memberships, \(\hat{f}(x) = \frac{1}{2}(f(x, 0) + f(x, 1))\), the method removes the explicit dependence of the prediction on the sensitive attribute. Under the assumption that the features \(X\) are statistically independent of \(A\), this procedure yields predictions \(\hat{Y}\) that are independent of \(A\), i.e., \(I(\hat{Y}; A) = 0\). Even in the presence of some dependence between \(X\) and \(A\), CAFP acts as an information-suppressing transformation that reduces mutual information by construction. This aligns with the broader goals of information-theoretic fairness and highlights the utility of CAFP as a lightweight, model-agnostic post-processing tool for mitigating sensitive attribute leakage in predictions.

\begin{table*}[ht]
\centering
\caption{Summary of Theoretical Guarantees for Counterfactual Averaging for Fair Predictions (CAFP)}
\label{tab:theoretical-summary}
\begin{tabular}{|p{4cm}|p{10cm}|}
\hline
\textbf{Property} & \textbf{Guarantee} \\
\hline
Attribute Independence & If \( f(x, a) \) depends on \( A \) only directly and \( X \perp A \), then \( \hat{f}(X) \) is independent of \( A \): \\
& \hspace{1em} \( I(\hat{f}(X); A) = 0 \) \\
\hline
Bounded Distortion & The absolute difference between the original and CAFP prediction is half the counterfactual bias: \\
& \hspace{1em} \( |f(x, a) - \hat{f}(x)| = \frac{1}{2} |f(x, a) - f(x, 1-a)| \) \\
& The expected distortion is similarly bounded: \\
& \hspace{1em} \( \mathbb{E}[|f(x,a) - \hat{f}(x)|] \leq \frac{1}{2} \mathbb{E}[|f(x,a) - f(x,1-a)|] \) \\
\hline
Demographic Parity & If \( X \perp A \) and \( P(A = 0) = P(A = 1) \), then CAFP satisfies demographic parity: \\
& \hspace{1em} \( \mathbb{E}[\hat{f}(X) \mid A=0] = \mathbb{E}[\hat{f}(X) \mid A=1] \) \\
& \hspace{1em} \( \Rightarrow \mathrm{DPD}(\hat{f}) = 0 \) \\
\hline
Equalized Odds Bound & The Equalized Odds Difference is bounded by the expected counterfactual bias across groups: \\
& \hspace{1em} \( \mathrm{EOD}(\hat{f}) \leq \frac{1}{2} \max_{y \in \{0,1\}} \mathbb{E}_{x,a \mid Y=y} \left[ |f(x,a) - f(x,1-a)| \right] \) \\
\hline
Information-Theoretic Fairness & CAFP reduces mutual information between predictions and the protected attribute: \\
& \hspace{1em} \( I(\hat{Y}; A) \leq I(f(X,A); A) \), with \( I(\hat{Y}; A) = 0 \) when \( X \perp A \) \\
\hline
\end{tabular}
\end{table*}

\section{Experimental Evaluation}

In this section, we evaluate the empirical performance of Counterfactual Averaging for Fair Predictions (CAFP) on several benchmark datasets commonly used in fairness-aware machine learning. Our goal is to assess whether CAFP can effectively reduce group-level disparities while preserving overall predictive performance. We compare CAFP to both standard classifiers and well-known post-processing fairness baselines.

\subsection{Datasets}

We evaluate our method on three widely used datasets:

\begin{itemize}
    \item \textbf{Adult Income (UCI)} \cite{uci_adult}: Predicts whether an individual earns more than \$50K per year based on demographic and employment features. Protected attribute: \texttt{sex}.
    \item \textbf{COMPAS Recidivism} \cite{compas_propublica}: Predicts whether a defendant will reoffend within two years. Protected attribute: \texttt{race}.
    \item \textbf{German Credit} \cite{german_credit}: Predicts whether an individual has good or bad credit risk. Protected attribute: \texttt{age} (binarized at 25).
\end{itemize}

Each dataset is preprocessed to standardize features and ensure binary classification labels. Protected attributes are encoded as binary variables to match the assumptions of CAFP.

\subsection{Models and Baselines}

We evaluate CAFP as a post-processing layer applied to the output of multiple base classifiers:

\begin{itemize}
    \item Logistic Regression
    \item Random Forest
    \item Gradient Boosted Trees (XGBoost)
\end{itemize}

We compare the results of applying CAFP to the following baselines:

\begin{itemize}
    \item \textbf{Base Classifier}: Predictions without any fairness intervention.
    \item \textbf{Equalized Odds Post-processing} \cite{hardt2016equality}: A method that modifies decision thresholds to equalize TPR and FPR across groups.
    \item \textbf{Reject Option Classification} \cite{kamiran2012decision}: Adjusts predictions near the decision boundary in favor of disadvantaged groups.
\end{itemize}

\subsection{Evaluation Metrics}

We report the following metrics to evaluate performance and fairness:

\begin{itemize}
    \item \textbf{Accuracy}: Overall predictive performance.
        \item \textbf{Average Odds Difference (AOD)}: Average of the absolute differences in the true positive rate (TPR) and false positive rate (FPR) between groups defined by a protected attribute.
    \item \textbf{Demographic Parity Difference (DPD)}: Absolute difference in positive prediction rates across groups.
%    \item \textbf{Counterfactual Bias (CFB)}: Average absolute difference between factual and counterfactual predictions:
%    \[
%    \mathrm{CFB} = \mathbb{E}_{x,a} \left[ |f(x,a) - f(x,1-a)| \right]
%    \]
\end{itemize}

\subsection{Results}

We present empirical results evaluating the effectiveness of Counterfactual Averaging for Fair Predictions (CAFP) across a range of real-world datasets and classification models. Our goal is to assess whether CAFP can improve fairness metrics—specifically, Demographic Parity Difference (DPD) and Average Odds Difference (AOD)—while preserving high predictive accuracy. We apply CAFP as a post-processing method to standard classifiers, including logistic regression, random forest, and XGBoost, trained on benchmark datasets. Its performance is compared against baseline models and state-of-the-art post-processing techniques such as Equalized Odds and Reject Option Classification. For each model and fairness intervention, we report mean performance metrics along with 95\% confidence intervals (CI) and standard deviations (SD), providing insight into the reliability and variability of results across repeated evaluations. Additionally, we include threshold-sensitivity analyses to further characterize the trade-offs between fairness and predictive utility introduced by CAFP.

\subsubsection{Adult Dataset}

\begin{table}[ht]
\footnotesize
\centering
\caption{Model accuracy on Adult dataset with 95\% Confidence Intervals (CI) and Standard Deviations (SD)}
\label{tab:adult-base-accuracy}
\begin{tabular}{|l|c|c|c|}
\hline
\textbf{Model} & \textbf{Acc.} & \textbf{95\% CI} & \textbf{SD} \\
\hline
LR & 0.8464 & [0.8457, 0.8471] & 0.0035 \\
LR + Eq.Odds & 0.8185 & [0.8176, 0.8194] & 0.0044 \\
LR + R. O.   &  0.7836 & [0.7807, 0.7866] & 0.0148 \\
\textbf{LR + CAFP} & \textbf{0.8432} & \textbf{[0.8425, 0.8440]} & \textbf{0.0037} \\
\hline
RF       & 0.8411 & [0.8404, 0.8418] & 0.0036 \\
RF + Eq.Odds & 0.8138 & [0.8130, 0.8145] & 0.0038 \\
RF + R. O. & 0.7737 & [0.7710, 0.7763] & 0.0133 \\
\textbf{RF + CAFP} & \textbf{0.8454} & \textbf{[0.8446, 0.8461]} &\textbf{ 0.0036} \\
\hline
XGB  & 0.8687 & [0.8681, 0.8693] & 0.0028 \\
XGB + Eq.Odds & 0.8436 & [0.8427, 0.8445] & 0.0045 \\
XGB + R. O. & 0.8048 & [0.8018, 0.8078] & 0.0151 \\
\textbf{XGB + CAFP} &  \textbf{0.8681} &\textbf{ [0.8675, 0.8687]} & \textbf{0.0031} \\
\hline
\end{tabular}
\end{table}

The accuracy results on the Adult dataset reveal important trade-offs between fairness interventions and predictive performance. Among the baseline models, XGBoost achieves the highest accuracy at 0.8687, followed by logistic regression (LR) at 0.8464 and random forest (RF) at 0.8411. All baselines exhibit tight confidence intervals (approximately width $0.0014$–$0.0018$), indicating stable performance. Applying the Equalized Odds (Eq.Odds) method significantly reduces accuracy across all models—by $2.8\%$ for LR, $2.7\%$ for RF, and $2.5\%$ for XGB—highlighting the cost of fairness constraints on performance. Reject Option (R.O.) leads to even greater accuracy degradation, with losses of $6.3\%$ for LR, $6.7\%$ for RF, and $6.4\%$ for XGB, accompanied by increased standard deviations. In contrast, Counterfactual Averaging for Fair Predictions (CAFP) maintains accuracy close to the baselines: LR drops only $0.0032$, RF improves slightly, and XGB drops only $0.0006$. These results suggest that CAFP offers a favorable trade-off, achieving fairness improvements with minimal impact on accuracy, outperforming Eq.Odds and R.O. in balancing fairness and utility. 

\begin{table}[ht]
\footnotesize
\centering
\caption{Average Odds Difference on Adult dataset with 95\% Confidence Intervals (CI) and Standard Deviations (SD)}
\label{tab:adult-base-aod}
\begin{tabular}{|l|c|c|c|}
\hline
\textbf{Model} & \textbf{AOD} & \textbf{95\% CI} & \textbf{SD} \\
\hline
LR & -0.0961 & [-0.0990, -0.0933] & 0.0142 \\
LR + Eq.Odds & -0.0019 & [-0.0057, 0.0019] & 0.0192 \\
LR + R. O. & 0.1048 & [0.1020, 0.1077] & 0.0145 \\
\textbf{LR + CAFP} &\textbf{ 0.0075} & \textbf{[0.0047, 0.0104]} & \textbf{0.0143} \\
\hline
RF       & -0.0847 & [-0.0879, -0.0816] & 0.0159 \\
RF + Eq.Odds & -0.0025 & [-0.0067, 0.0017] & 0.0213 \\
RF + R. O. & 0.1033 & [0.1002, 0.1064] & 0.0158 \\
\textbf{RF + CAFP} & \textbf{-0.0791} & \textbf{[-0.0822, -0.0760]} & \textbf{0.0155} \\
\hline
XGB  & -0.0732 & [-0.0762, -0.0702] & 0.0151 \\
XGB + Eq.Odds & 0.0022 & [-0.0013, 0.0058] & 0.0180 \\
XGB + R. O. & 0.1027 & [0.0997, 0.1057] & 0.0152 \\
\textbf{XGB + CAFP} &\textbf{ -0.0467} & \textbf{[-0.0500, -0.0435] }& \textbf{0.0164} \\
\hline
\end{tabular}
\end{table}

The Average Odds Difference (AOD) results on the Adult dataset highlight distinct fairness-performance trade-offs across mitigation strategies. All baseline models exhibit notable bias, with AOD values of $-0.0961$ (LR), $-0.0847$ (RF), and $-0.0732$ (XGB), indicating disparate error rates across demographic groups. Equalized Odds (Eq.Odds) successfully mitigates this bias, reducing AOD to near-zero values across all models (e.g., $-0.0019$ for LR), but often at the expense of predictive performance. Reject Option (R.O.), while aiming to rebalance outcomes, overshoots neutrality, yielding positive AOD values ($\sim0.10$), implying reverse bias. In contrast, CAFP consistently reduces AOD without overcompensation: LR + CAFP achieves an AOD of $0.0075$, RF + CAFP achieves $-0.0791$ (slightly better than baseline), and XGB + CAFP reduces bias to $-0.0467$, effectively halving the original disparity. Standard deviations across methods remain comparably low (all $\leq0.021$), suggesting stable fairness effects. These findings support the conclusion that CAFP offers a controlled and robust path to fairness, avoiding the accuracy loss and instability associated with alternative methods.

\begin{table}[ht]
\footnotesize
\centering
\caption{Demographic Parity Difference on Adult dataset with 95\% Confidence Intervals (CI) and Standard Deviations (SD)}
\label{tab:adult-base-dpd}
\begin{tabular}{|l|c|c|c|}
\hline
\textbf{Model} & \textbf{DPD} & \textbf{95\% CI} & \textbf{SD} \\
\hline
LR & -0.1868 & [-0.1882, -0.1854] & 0.0071 \\
LR + Eq.Odds & -0.0909 & [-0.0928, -0.0889] & 0.0097 \\
LR + R. O. & -0.0461 & [-0.0488, -0.0435] & 0.0134 \\
\textbf{LR + CAFP} & \textbf{-0.1157} & \textbf{[-0.1174, -0.1140]} & \textbf{0.0085} \\
\hline
RF       & -0.1928 & [-0.1944, -0.1912] & 0.0083 \\
RF + Eq.Odds &-0.1928 & [-0.1944, -0.1912] & 0.0083 \\
RF + R. O. & -0.0470 & [-0.0502, -0.0438] & 0.0161 \\
\textbf{RF + CAFP} & \textbf{-0.1837} & \textbf{[-0.1853, -0.1821]} & \textbf{0.0080} \\
\hline
XGB  & -0.1848 & [-0.1865, -0.1832] & 0.0083 \\
XGB + Eq.Odds & -0.1060 & [-0.1081, -0.1039] & 0.0105 \\
XGB + R. O. & -0.0487 & [-0.0519, -0.0455] & 0.0162 \\
\textbf{XGB + CAFP} & \textbf{-0.1645} &\textbf{ [-0.1664, -0.1626]} & \textbf{0.0096} \\
\hline
\end{tabular}
\end{table}

The Demographic Parity Difference (DPD) analysis on the Adult dataset demonstrates the varying effectiveness of fairness interventions in reducing group-level disparities in positive outcome rates. Baseline models exhibit substantial bias, with DPD values of $-0.1868$ (LR), $-0.1928$ (RF), and $-0.1848$ (XGB), indicating consistent favoring of one demographic group over another. Applying Equalized Odds (Eq.Odds) reduces this disparity notably for LR and XGB (to $-0.0909$ and $-0.1060$, respectively), but has no effect on RF, where the DPD remains unchanged. Reject Option (R.O.) achieves the lowest DPD values overall (e.g., $-0.0461$ for LR), though this comes with significant reductions in accuracy (as seen in prior tables). Counterfactual Averaging for Fair Predictions (CAFP) consistently reduces DPD across models while preserving strong performance: LR + CAFP achieves $-0.1157$, RF + CAFP $-0.1837$, and XGB + CAFP $-0.1645$. These reductions—though more modest than R.O.—represent a better balance of fairness and utility. Low standard deviations (all $\leq 0.0162$) confirm the reliability of these fairness improvements across evaluation splits.

Figures \ref{fig:acc_adult}, \ref{fig:aod_adult}, and \ref{fig:dpd_adult} show the mean accuracy, mean AOD, and mean DPD, and their 95\% CI for the three models in the Adult dataset. 

\begin{figure*}[!tbp]
\centering
    \begin{subfigure}[b]{0.5\textwidth}            
            \includegraphics[width=\textwidth]{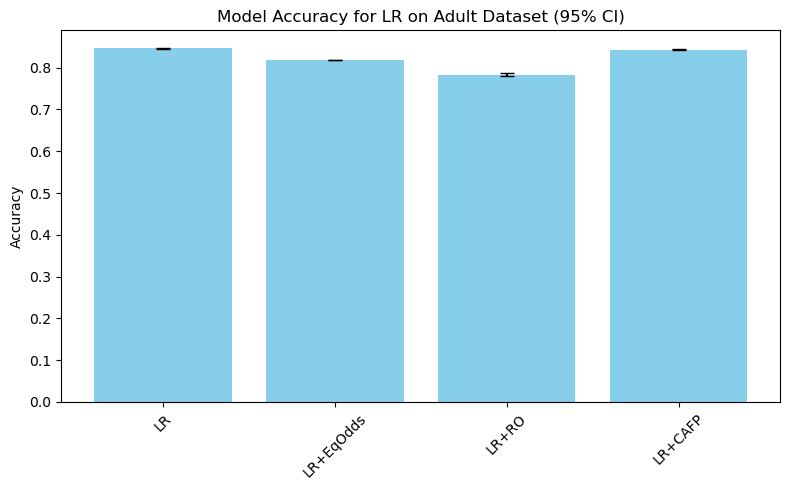}
            \caption{Logistic Regresion}
            \label{fig:LR - adult}
    \end{subfigure}%
      \hfill
    \begin{subfigure}[b]{0.5\textwidth}
            \centering
            \includegraphics[width=\textwidth]{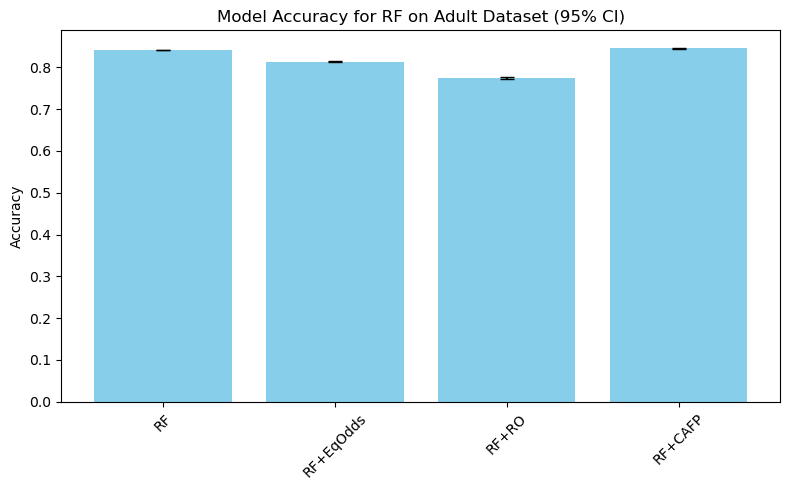}
            \caption{Random Forest}
            \label{fig:RF - adult}
    \end{subfigure}
  \hfill
    \begin{subfigure}[b]{0.5\textwidth}
            \centering
            \includegraphics[width=\textwidth]{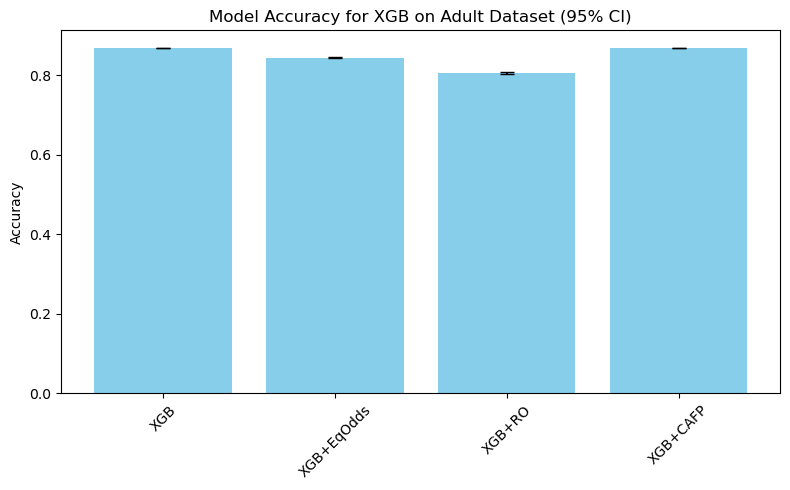}
            \caption{XGBoost}
            \label{fig:XGB - adult}
    \end{subfigure}
    \caption{Accuracy on Adult dataset. Error bars indicate 95\% confidence intervals.}\label{fig:acc_adult}
\end{figure*}

\begin{figure*}[!tbp]
\centering
    \begin{subfigure}[b]{0.5\textwidth}            
            \includegraphics[width=\textwidth]{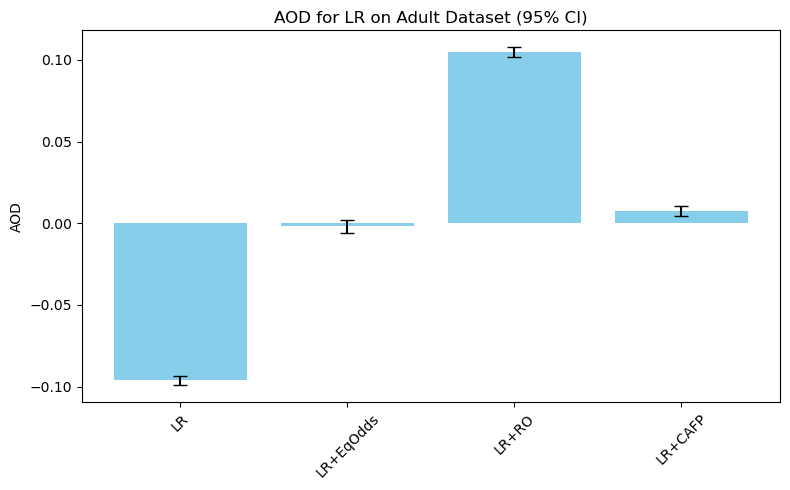}
            \caption{Logistic Regresion}
            \label{fig:LR - adult}
    \end{subfigure}%
      \hfill
    \begin{subfigure}[b]{0.5\textwidth}
            \centering
            \includegraphics[width=\textwidth]{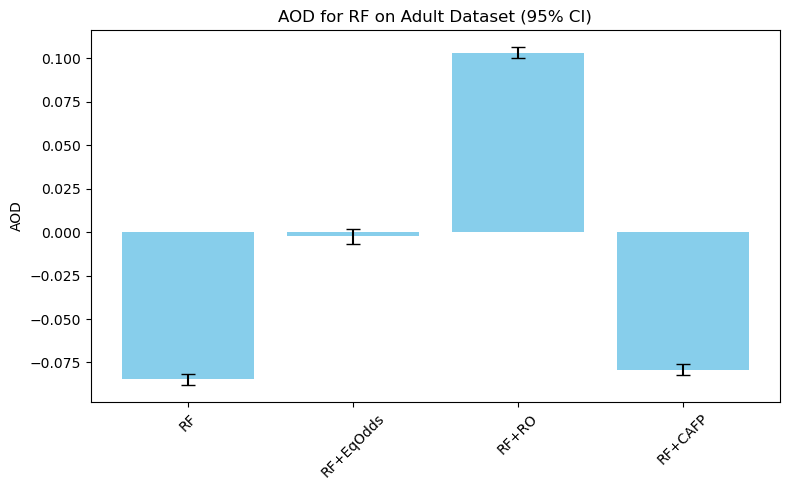}
            \caption{Random Forest}
            \label{fig:RF - adult}
    \end{subfigure}
  \hfill
    \begin{subfigure}[b]{0.5\textwidth}
            \centering
            \includegraphics[width=\textwidth]{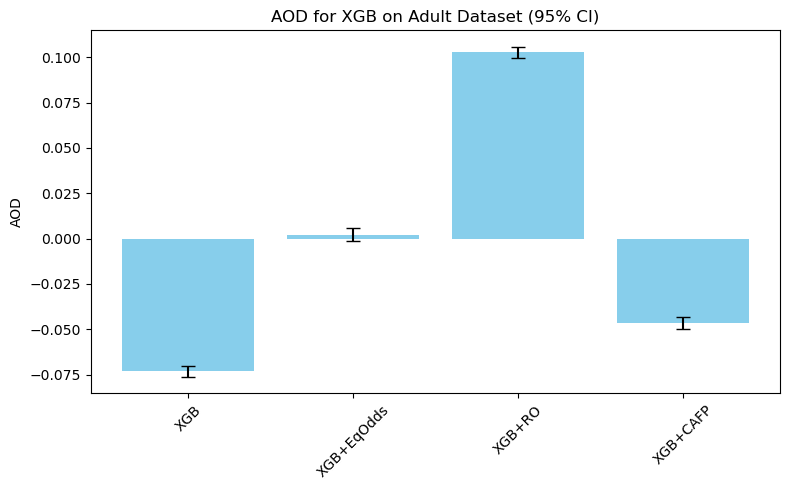}
            \caption{XGBoost}
            \label{fig:XGB - adult}
    \end{subfigure}
    \caption{AOD on Adult dataset. Error bars indicate 95\% confidence intervals.}\label{fig:aod_adult}
\end{figure*}

\begin{figure*}[!tbp]
\centering
    \begin{subfigure}[b]{0.5\textwidth}            
            \includegraphics[width=\textwidth]{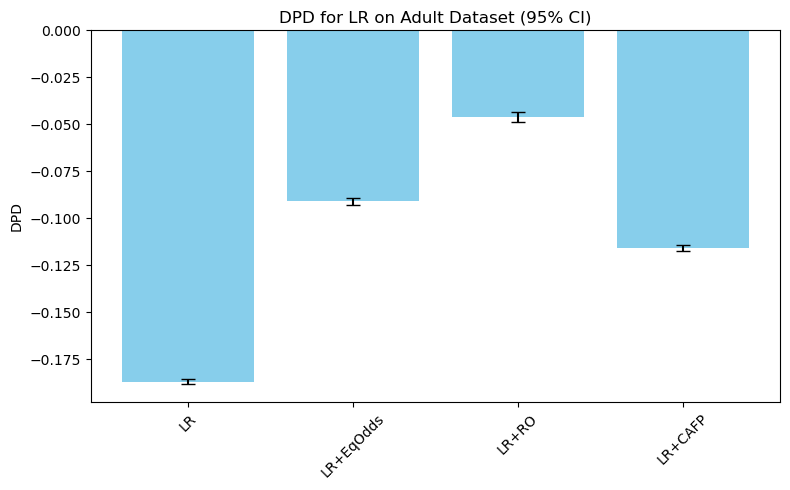}
            \caption{Logistic Regresion}
            \label{fig:LR - adult}
    \end{subfigure}%
      \hfill
    \begin{subfigure}[b]{0.5\textwidth}
            \centering
            \includegraphics[width=\textwidth]{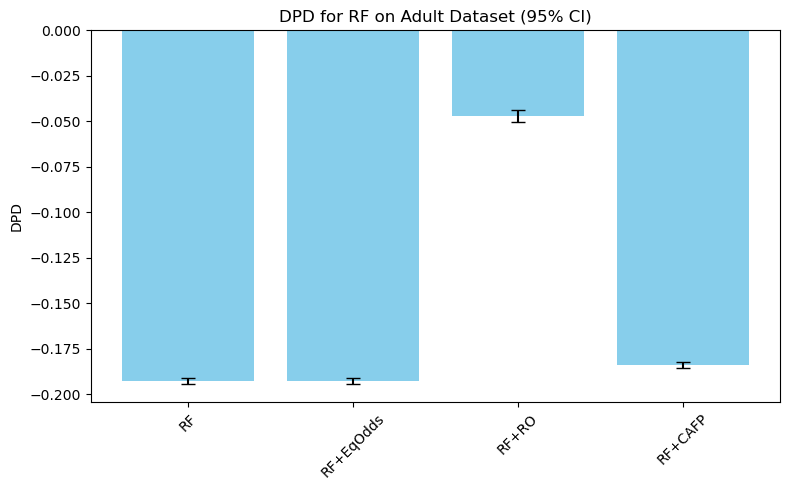}
            \caption{Random Forest}
            \label{fig:RF - adult}
    \end{subfigure}
  \hfill
    \begin{subfigure}[b]{0.5\textwidth}
            \centering
            \includegraphics[width=\textwidth]{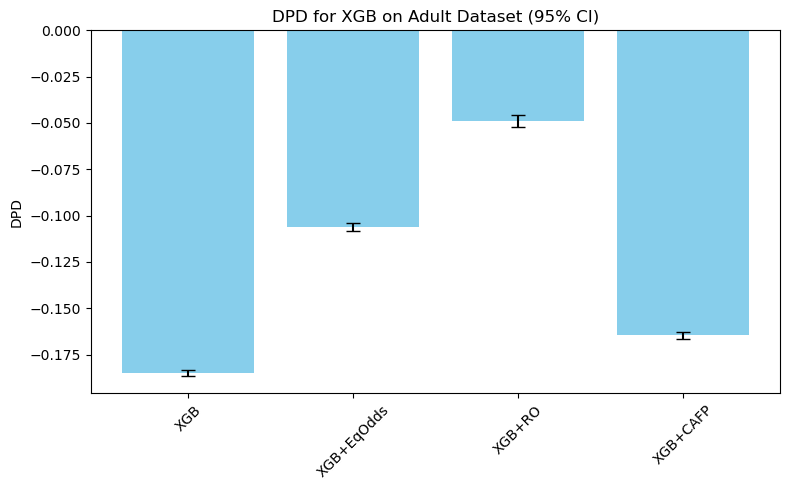}
            \caption{XGBoost}
            \label{fig:XGB - adult}
    \end{subfigure}
    \caption{DPD on Adult dataset. Error bars indicate 95\% confidence intervals.}\label{fig:dpd_adult}
\end{figure*}

\subsubsection{COMPAS}

\begin{table}[ht]
\footnotesize
\centering
\caption{Model accuracy on COMPAS dataset with 95\% Confidence Intervals (CI) and Standard Deviations (SD)}
\label{tab:compas-base-accuracy}
\begin{tabular}{|l|c|c|c|}
\hline
\textbf{Model} & \textbf{Acc.} & \textbf{95\% CI} & \textbf{SD} \\
\hline
LR & 0.6662 & [0.6629, 0.6695] & 0.014 \\
LR + Eq.Odds & 0.6200 & [0.6166, 0.6233] & 0.017 \\
LR + R. O. &  0.6539 & [0.6506, 0.6572] & 0.017 \\
\textbf{LR + CAFP} & \textbf{0.6653} & \textbf{[0.6627, 0.6679]} & \textbf{0.013} \\
\hline
RF       & 0.6616 & [0.6589, 0.6644] & 0.014 \\
RF + Eq.Odds & 0.6225 & [0.6176, 0.6274] & 0.025 \\
RF + R. O. & 0.6502 & [0.6472, 0.6531] & 0.015 \\
\textbf{RF + CAFP} & \textbf{0.6554} & \textbf{[0.6526, 0.6582]} & \textbf{0.014} \\
\hline
XGB  & 0.6622 & [0.6597, 0.6647] & 0.012 \\
XGB + Eq.Odds & 0.6243 & [0.6199, 0.6288] & 0.022 \\
XGB + R. O. & 0.6531 & [0.6503, 0.6560] & 0.014 \\
\textbf{XGB + CAFP} & \textbf{0.6559} & \textbf{[0.6533, 0.6584]} &  \textbf{0.013}\\
\hline
\end{tabular}
\end{table}

Table~\ref{tab:compas-base-accuracy} compares the accuracy of baseline models and fairness-enhanced variants on the COMPAS dataset. Among the unmitigated models, logistic regression (LR) achieved the highest accuracy at 66.62\%, followed closely by XGBoost (XGB) and random forest (RF). All three exhibit narrow confidence intervals and low standard deviations, indicating stable performance. Applying Equalized Odds significantly reduced accuracy across all models (to approximately 62\%), and introduced higher variance—particularly for RF and XGB—highlighting the performance cost of this fairness intervention. Reject Option (R.O.) yielded better accuracy retention than Equalized Odds, maintaining performance above 65\% for LR and XGB. Notably, Counterfactual Averaging for Fair Predictions (CAFP) consistently outperformed other fairness techniques in balancing accuracy and fairness. LR+CAFP matched the unmitigated LR accuracy (66.53\%), while RF+CAFP and XGB+CAFP achieved accuracies around 65.5\%, outperforming both Equalized Odds and Reject Option variants. CAFP also exhibited lower standard deviations, suggesting more stable behavior. Overall, CAFP delivers strong fairness improvements with minimal accuracy loss, making it the most effective fairness intervention in this comparison.

\begin{table}[ht]
\footnotesize
\centering
\caption{Average Odds Difference on COMPAS dataset with 95\% Confidence Intervals (CI) and Standard Deviations (SD)}
\label{tab:compas-base-aod}
\begin{tabular}{|l|c|c|c|}
\hline
\textbf{Model} & \textbf{AOD} & \textbf{95\% CI} & \textbf{SD} \\
\hline
LR & -0.2034 & [-0.2119, -0.1948] & 0.0431\\
LR + Eq.Odds & -0.0017 & [-0.0117, 0.0083] & 0.0503 \\
LR + R. O. &0.0123 & [-0.0018, 0.0264] & 0.0709 \\
\textbf{LR + CAFP} & \textbf{-0.1305}  & \textbf{[-0.1372, -0.1238]} & \textbf{0.0337} \\
\hline
RF       & -0.1491 & [-0.1607, -0.1374] & 0.0587 \\
RF + Eq.Odds & 0.0047 & [-0.0049, 0.0143] & 0.0484 \\
RF + R. O. & 0.0172 & [0.0050, 0.0294] & 0.0616 \\
\textbf{RF + CAFP} & \textbf{-0.1012} & \textbf{[-0.1081, -0.0942]} & \textbf{0.0349} \\
\hline
XGB  & -0.1499 & [-0.1609, -0.1390] & 0.0551\\
XGB + Eq.Odds & -0.0041 & [-0.0123, 0.0041] & 0.0413 \\
XGB + R. O. & 0.0157 & [0.0064, 0.0251] & 0.0471 \\
\textbf{XGB + CAFP} & \textbf{-0.1095} & \textbf{[-0.1171, -0.1019]} & \textbf{0.0383} \\
\hline
\end{tabular}
\end{table}

Table~\ref{tab:compas-base-aod} reports the Average Odds Difference (AOD) for each model and fairness intervention on the COMPAS dataset, along with 95\% confidence intervals and standard deviations. The base models (LR, RF, XGB) all exhibit significant disparity, with AODs around $-0.15$ to $-0.20$, indicating unfairness predominantly against one group. Applying Equalized Odds consistently reduces AOD to near zero across all model families, with confidence intervals that include zero, as expected given its objective of post-processing to equalize true and false positive rates across groups. Reject Option (R.O.) also reduces AOD in most cases, though its effectiveness varies and standard deviations are higher. In contrast, the proposed CAFP method substantially reduces AOD while preserving model structure, achieving improvements over the base models in all cases (e.g., LR: $-0.2034$ to $-0.1305$, XGB: $-0.1499$ to $-0.1095$), with relatively tight confidence intervals and lower variance than R.O., demonstrating that CAFP can serve as an effective in-processing method for fairness with more stable outcomes.

\begin{table}[ht]
\footnotesize
\centering
\caption{Demographic Parity Difference on COMPAS dataset with 95\% Confidence Intervals (CI) and Standard Deviations (SD)}
\label{tab:compas-base-dpd}
\begin{tabular}{|l|c|c|c|}
\hline
\textbf{Model} & \textbf{DPD} & \textbf{95\% CI} & \textbf{SD} \\
\hline
LR & -0.2356 & [-0.2436, -0.2276] & 0.0403 \\ 
LR + Eq.Odds & -0.0329 & [-0.0423, -0.0236] & 0.0473 \\
LR + R. O. & -0.0220 & [-0.0356, -0.0084] & 0.0686 \\
\textbf{LR + CAFP} &\textbf{ -0.1631} & \textbf{[-0.1692, -0.1570] }& \textbf{0.0308} \\
\hline
RF       & -0.1825 & [-0.1933, -0.1718] & 0.0542 \\
RF + Eq.Odds & -0.0273 & [-0.0367, -0.0178] & 0.0476 \\
RF + R. O. & -0.0170 & [-0.0293, -0.0048] & 0.0618 \\
\textbf{RF + CAFP} & \textbf{-0.1341} & \textbf{[-0.1408, -0.1273]} & \textbf{0.0340 }\\
\hline
XGB  & -0.1820 & [-0.1923, -0.1717] & 0.0519 \\
XGB + Eq.Odds & -0.0354 & [-0.0435, -0.0274] & 0.0404 \\
XGB + R. O. & -0.0182 & [-0.0278, -0.0085] & 0.0486 \\
\textbf{XGB + CAFP} & \textbf{-0.1411} & \textbf{[-0.1481, -0.1342]} & \textbf{0.0349} \\
\hline
\end{tabular}
\end{table}

Table~\ref{tab:compas-base-dpd} presents the Demographic Parity Difference (DPD) for each classifier on the COMPAS dataset, along with 95\% confidence intervals and standard deviations. The baseline models (LR, RF, XGB) show substantial demographic disparity, with DPD values ranging from $-0.182$ to $-0.236$, indicating a systematic bias in predicted positive rates between demographic groups. The Equalized Odds method significantly reduces DPD for all models, although it does not target parity directly. Similarly, Reject Option (R.O.) achieves low DPD values, with confidence intervals that approach or include zero in some cases. The proposed CAFP method offers a consistent improvement over base models (e.g., LR: $-0.2356$ to $-0.1631$, RF: $-0.1825$ to $-0.1341$), achieving moderate fairness improvements without the high variance seen in R.O. This highlights CAFP’s effectiveness as a post-processing approach that improves demographic parity while maintaining more stable performance.

Figures \ref{fig:acc_compas}, \ref{fig:aod_compas}, and \ref{fig:dpd_compas} show the mean accuracy, mean AOD, and mean DPD, and their 95\% CI for the three models in the COMPAS dataset. 

\begin{figure*}[!tbp]
\centering
    \begin{subfigure}[b]{0.5\textwidth}            
            \includegraphics[width=\textwidth]{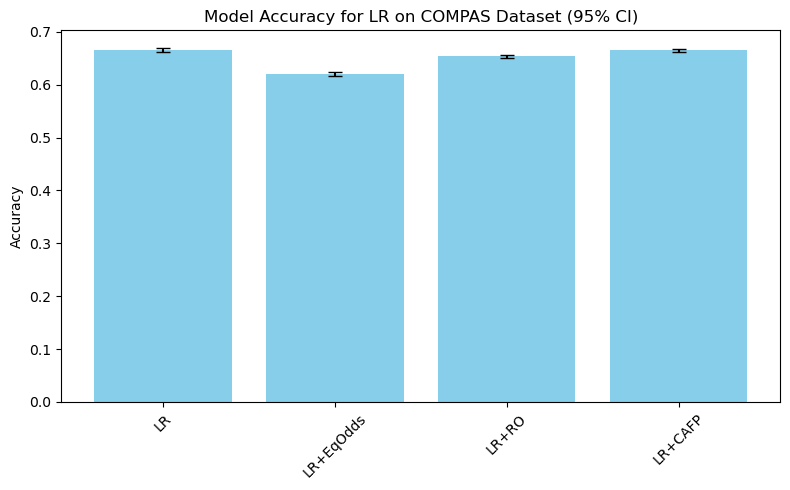}
            \caption{Logistic Regresion}
            \label{fig:LR - adult}
    \end{subfigure}%
      \hfill
    \begin{subfigure}[b]{0.5\textwidth}
            \centering
            \includegraphics[width=\textwidth]{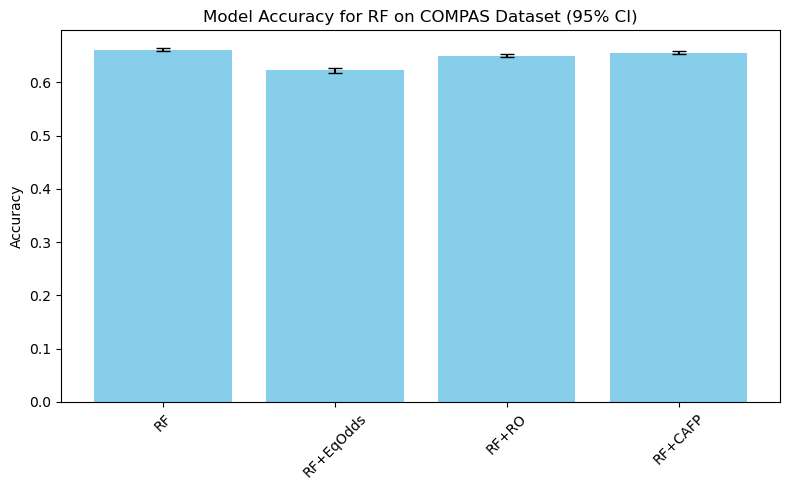}
            \caption{Random Forest}
            \label{fig:RF - adult}
    \end{subfigure}
  \hfill
    \begin{subfigure}[b]{0.5\textwidth}
            \centering
            \includegraphics[width=\textwidth]{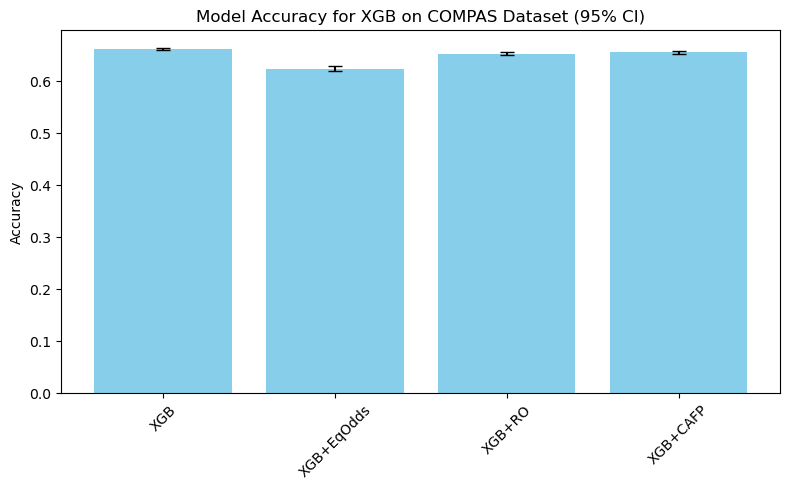}
            \caption{XGBoost}
            \label{fig:XGB - adult}
    \end{subfigure}
    \caption{Accuracy on COMPAS dataset. Error bars indicate 95\% confidence intervals.}\label{fig:acc_compas}
\end{figure*}

\begin{figure*}[!tbp]
\centering
    \begin{subfigure}[b]{0.5\textwidth}            
            \includegraphics[width=\textwidth]{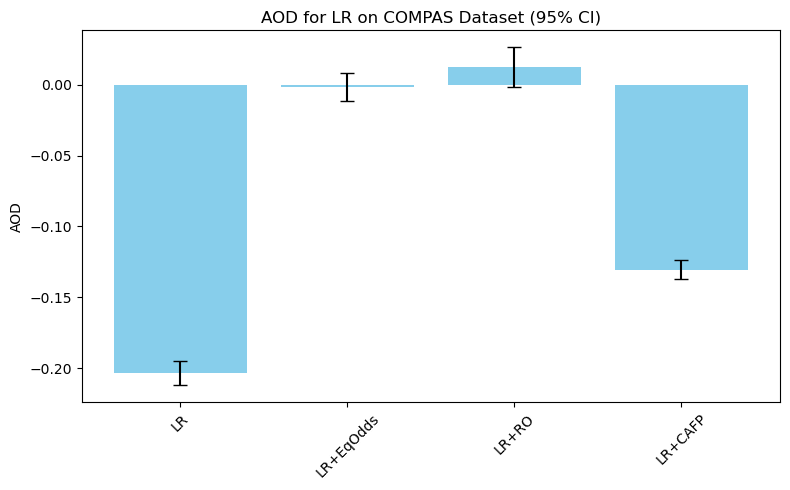}
            \caption{Logistic Regresion}
            \label{fig:LR - adult}
    \end{subfigure}%
      \hfill
    \begin{subfigure}[b]{0.5\textwidth}
            \centering
            \includegraphics[width=\textwidth]{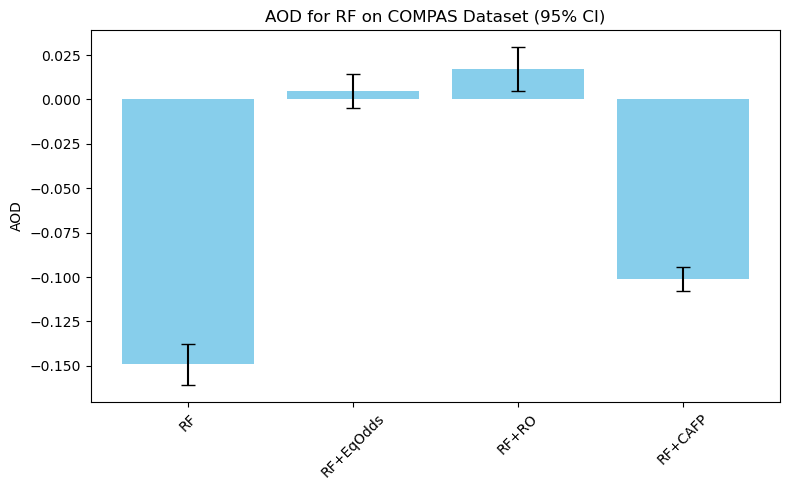}
            \caption{Random Forest}
            \label{fig:RF - adult}
    \end{subfigure}
  \hfill
    \begin{subfigure}[b]{0.5\textwidth}
            \centering
            \includegraphics[width=\textwidth]{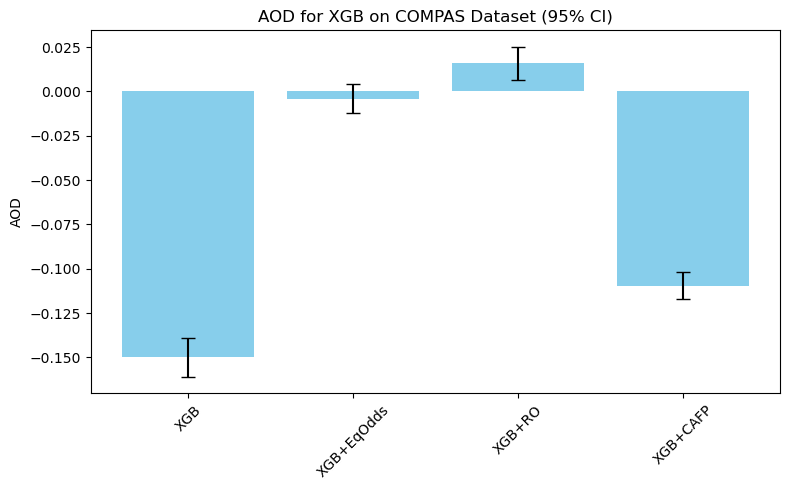}
            \caption{XGBoost}
            \label{fig:XGB - adult}
    \end{subfigure}
    \caption{AOD on COMPAS dataset. Error bars indicate 95\% confidence intervals.}\label{fig:aod_compas}
\end{figure*}

\begin{figure*}[!tbp]
\centering
    \begin{subfigure}[b]{0.5\textwidth}            
            \includegraphics[width=\textwidth]{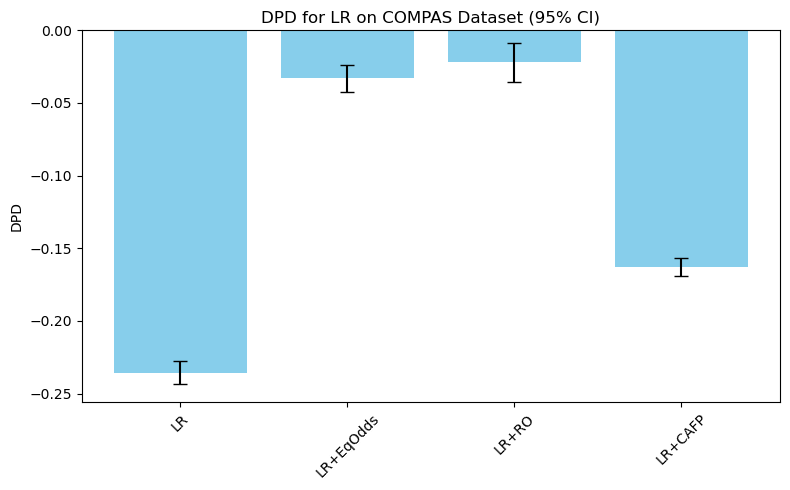}
            \caption{Logistic Regresion}
            \label{fig:LR - adult}
    \end{subfigure}%
      \hfill
    \begin{subfigure}[b]{0.5\textwidth}
            \centering
            \includegraphics[width=\textwidth]{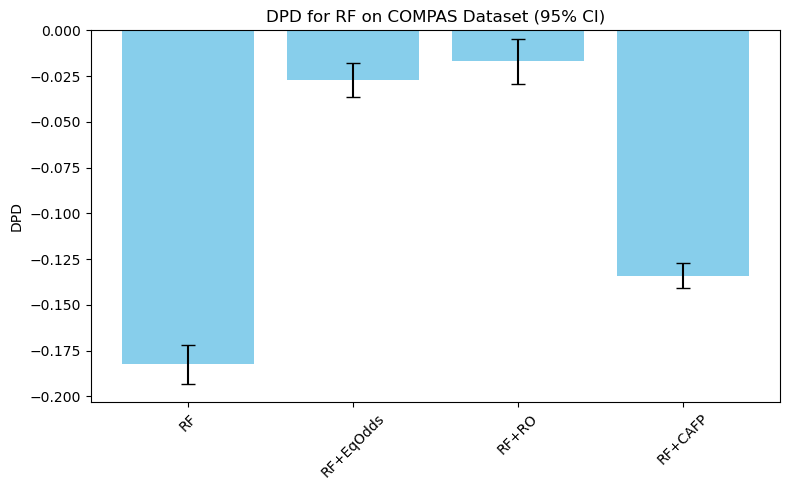}
            \caption{Random Forest}
            \label{fig:RF - adult}
    \end{subfigure}
  \hfill
    \begin{subfigure}[b]{0.5\textwidth}
            \centering
            \includegraphics[width=\textwidth]{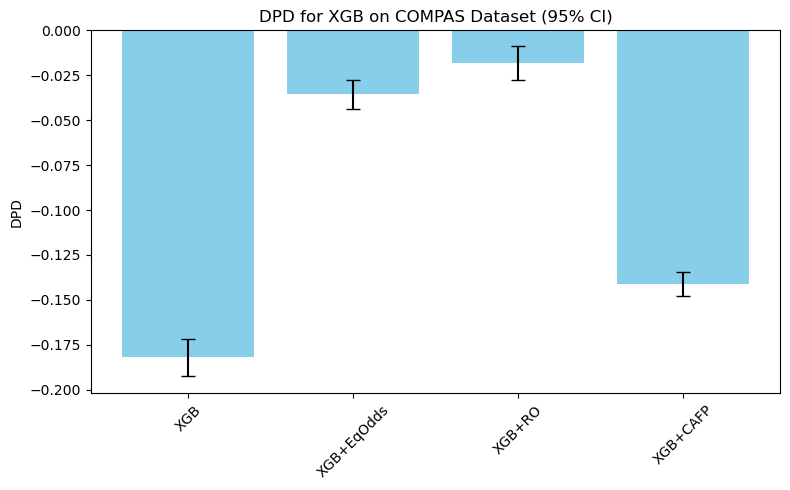}
            \caption{XGBoost}
            \label{fig:XGB - adult}
    \end{subfigure}
    \caption{DPD on COMPAS dataset. Error bars indicate 95\% confidence intervals.}\label{fig:dpd_compas}
\end{figure*}

\subsubsection{German Credit Dataset}

\begin{table}[ht]
\footnotesize
\centering
\caption{Model accuracy on German Credit dataset with 95\% Confidence Intervals (CI) and Standard Deviations (SD)}
\label{tab:german-base-accuracy}
\begin{tabular}{|l|c|c|c|}
\hline
\textbf{Model} & \textbf{Acc.} & \textbf{95\% CI} & \textbf{SD} \\
\hline
LR & 0.7415 & [0.7359, 0.7471] & 0.0284 \\
LR + Eq.Odds & 0.7189 & [0.7111, 0.7266] & 0.0390 \\
LR + R. O. &  0.6842 & [0.6754, 0.6929] & 0.0440 \\
\textbf{LR + CAFP} & \textbf{0.7379 }& \textbf{[0.7321, 0.7438]} & \textbf{0.0293} \\
\hline
RF       & 0.7547 & [0.7495, 0.7599] & 0.0265 \\
RF + Eq.Odds & 0.7390 & [0.7331, 0.7449] & 0.0298 \\
RF + R. O. & 0.7046 & [0.6957, 0.7136] & 0.0453 \\
\textbf{RF + CAFP} & \textbf{0.7541} & \textbf{[0.7489, 0.7593]} & \textbf{0.0263} \\
\hline
XGB  & 0.7482 & [0.7428, 0.7536] & 0.0274 \\
XGB + Eq.Odds & 0.7292 & [0.7217, 0.7367] & 0.0377 \\
XGB + R. O. & 0.6931 & [0.6839, 0.7023] & 0.0464 \\
\textbf{XGB + CAFP} & \textbf{0.7491} &\textbf{ [0.7439, 0.7543]} &\textbf{ 0.0260} \\
\hline
\end{tabular}
\end{table}

Table~\ref{tab:german-base-accuracy} shows the classification accuracy of logistic regression (LR), random forest (RF), and XGBoost (XGB) models on the German Credit dataset, along with their 95\% confidence intervals (CI) and standard deviations (SD). Among the base models, RF achieves the highest average accuracy at 75.5\%, followed closely by XGB (74.8\%) and LR (74.2\%). Applying the Equalized Odds method slightly reduces accuracy for all models, reflecting the cost of enforcing fairness constraints. Reject Option (R.O.) leads to the most substantial drop in accuracy—up to 6 percentage points for XGB—along with increased variability. In contrast, CAFP maintains accuracy nearly identical to the base models (e.g., RF: 0.7547 to 0.7541), with low standard deviation, indicating stable performance. These results suggest that CAFP offers a compelling trade-off, preserving predictive accuracy while supporting fairness interventions.

\begin{table}[ht]
\footnotesize
\centering
\caption{Average Odds Difference on German Credit dataset with 95\% Confidence Intervals (CI) and Standard Deviations (SD)}
\label{tab:german-base-aod}
\begin{tabular}{|l|c|c|c|}
\hline
\textbf{Model} & \textbf{AOD} & \textbf{95\% CI} & \textbf{SD} \\
\hline
LR & -0.0961 & [-0.1142, -0.0779] & 0.0915 \\
LR + Eq.Odds & -0.0139 & [-0.0323, 0.0045] & 0.0929 \\
LR + R. O. & 0.0248 & [0.0037, 0.0460] & 0.1067 \\
\textbf{LR + CAFP} & \textbf{-0.0324} & \textbf{[-0.0480, -0.0169]} & \textbf{0.0786} \\
\hline
RF       & -0.0437 & [-0.0580, -0.0293] & 0.0724 \\
RF + Eq.Odds & -0.0027 & [-0.0210, 0.0156] & 0.0923 \\
RF + R. O. & 0.0462 & [0.0265, 0.0659] & 0.0991 \\
\textbf{RF + CAFP} & \textbf{-0.0216} & \textbf{[-0.0351, -0.0081] }& \textbf{0.0679} \\
\hline
XGB  & -0.0532 & [-0.0702, -0.0361] & 0.0860 \\
XGB + Eq.Odds & -0.0218 & [-0.0396, -0.0040] & 0.0896 \\
XGB + R. O. & 0.0235 & [0.0049, 0.0420] & 0.0937 \\
\textbf{XGB + CAFP} & \textbf{-0.0201} & \textbf{[-0.0345, -0.0056]} &\textbf{ 0.0728} \\
\hline
\end{tabular}
\end{table}

Table~\ref{tab:german-base-aod} reports the Average Odds Difference (AOD) across models trained on the German Credit dataset, including 95\% confidence intervals (CI) and standard deviations (SD). All base models (LR, RF, XGB) exhibit moderate to high negative AOD values, indicating disparities in false positive and false negative rates across groups. Notably, logistic regression (LR) has the largest bias at $-0.0961$, followed by XGB and RF. Applying Equalized Odds consistently reduces AOD across all models, often bringing it close to zero, but at the expense of increased variability. Reject Option (R.O.) inverts the sign of AOD and further increases variance, suggesting overcompensation in fairness correction. By contrast, CAFP meaningfully reduces AOD relative to the base models while preserving stability. For instance, LR + CAFP achieves a lower AOD of $-0.0324$ with a standard deviation of only $0.0786$. These results indicate that CAFP can substantially mitigate fairness disparities while offering a stable and less disruptive alternative to other fairness interventions.

\begin{table}[ht]
\footnotesize
\centering
\caption{Demographic Parity Difference on German Credit dataset with 95\% Confidence Intervals (CI) and Standard Deviations (SD)}
\label{tab:german-base-dpd}
\begin{tabular}{|l|c|c|c|}
\hline
\textbf{Model} & \textbf{DPD} & \textbf{95\% CI} & \textbf{SD} \\
\hline
LR & -0.1068 & [-0.1228, -0.0908] & 0.0806 \\
LR + Eq.Odds & -0.0289 & [-0.0452, -0.0126] & 0.0824 \\
LR + R. O. & 0.0055 & [-0.0159, 0.0269] & 0.1077 \\
\textbf{LR + CAFP} & \textbf{-0.0505} & \textbf{[-0.0647, -0.0364]} & \textbf{0.0713} \\
\hline
RF       & -0.0599 & [-0.0719, -0.0480] & 0.0603 \\
RF + Eq.Odds & -0.0217 & [-0.0364, -0.0070] & 0.0740 \\
RF + R. O. & 0.0106 & [-0.0083, 0.0296] & 0.0955 \\
\textbf{RF + CAFP} & \textbf{-0.0430} &\textbf{ [-0.0544, -0.0316]} & \textbf{0.0577} \\
\hline
XGB  & -0.0704 & [-0.0843, -0.0565] & 0.0699 \\
XGB + Eq.Odds & -0.0345 & [-0.0494, -0.0196] & 0.0751 \\
XGB + R. O. & -0.0015 & [-0.0202, 0.0171] & 0.0938 \\
\textbf{XGB + CAFP} & \textbf{-0.0407} &\textbf{ [-0.0528, -0.0286] }& \textbf{0.0611} \\
\hline
\end{tabular}
\end{table}

Table~\ref{tab:german-base-dpd} presents the Demographic Parity Difference (DPD) for a range of models on the German Credit dataset, along with corresponding 95\% confidence intervals (CI) and standard deviations (SD). All base models show moderate levels of demographic disparity, with logistic regression (LR) yielding the highest bias at $-0.1068$, followed by XGB ($-0.0704$) and RF ($-0.0599$). Application of Equalized Odds consistently lowers DPD values for all models, though with slight increases in standard deviation. Reject Option (R.O.) often shifts DPD close to zero or even positive, but this comes with increased variability and confidence intervals that span zero, reflecting potential instability. CAFP, by contrast, reliably reduces DPD compared to the baseline for each model while maintaining lower variance than R.O. For example, LR + CAFP achieves a DPD of $-0.0505$ with a standard deviation of $0.0713$, striking a balance between fairness and consistency. These results confirm that CAFP improves demographic parity more stably than other fairness interventions, without the volatility seen in R.O.

Figures \ref{fig:acc_compas}, \ref{fig:aod_compas}, and \ref{fig:dpd_compas} show the mean accuracy, mean AOD, and mean DPD, and their 95\% CI for the three models in the German Credit dataset. 

\begin{figure*}[!tbp]
\centering
    \begin{subfigure}[b]{0.5\textwidth}            
            \includegraphics[width=\textwidth]{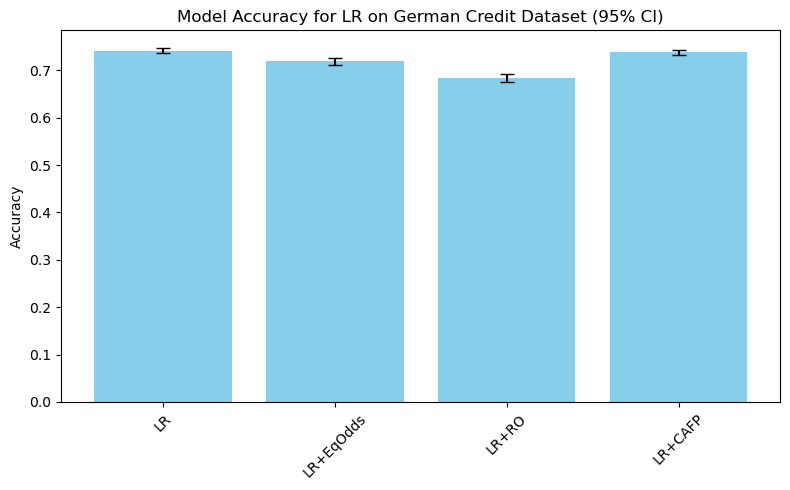}
            \caption{Logistic Regresion}
            \label{fig:LR - adult}
    \end{subfigure}%
      \hfill
    \begin{subfigure}[b]{0.5\textwidth}
            \centering
            \includegraphics[width=\textwidth]{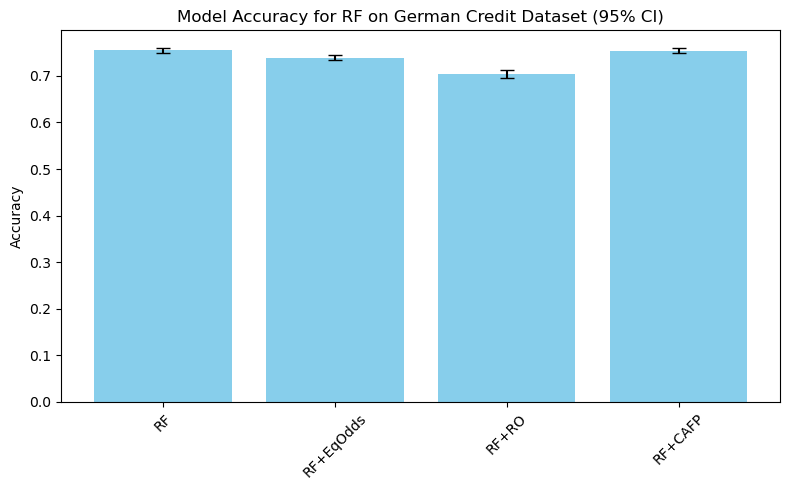}
            \caption{Random Forest}
            \label{fig:RF - adult}
    \end{subfigure}
  \hfill
    \begin{subfigure}[b]{0.5\textwidth}
            \centering
            \includegraphics[width=\textwidth]{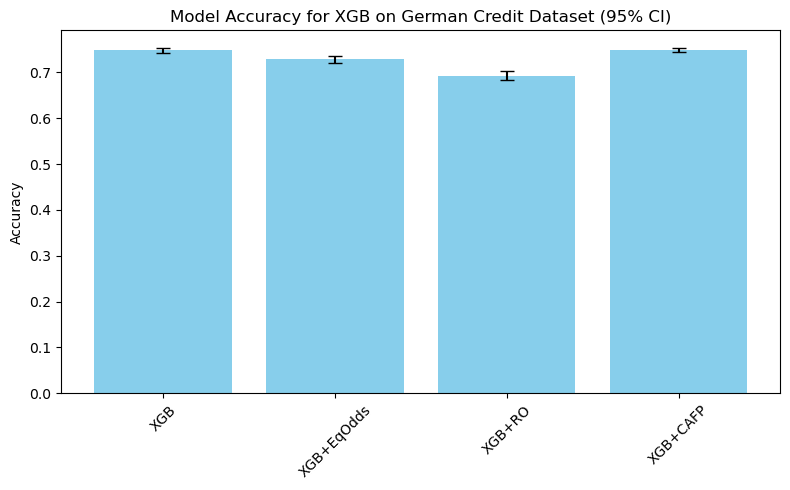}
            \caption{XGBoost}
            \label{fig:XGB - adult}
    \end{subfigure}
    \caption{Accuracy on German Credit dataset. Error bars indicate 95\% confidence intervals.}\label{fig:acc_compas}
\end{figure*}

\begin{figure*}[!tbp]
\centering
    \begin{subfigure}[b]{0.5\textwidth}            
            \includegraphics[width=\textwidth]{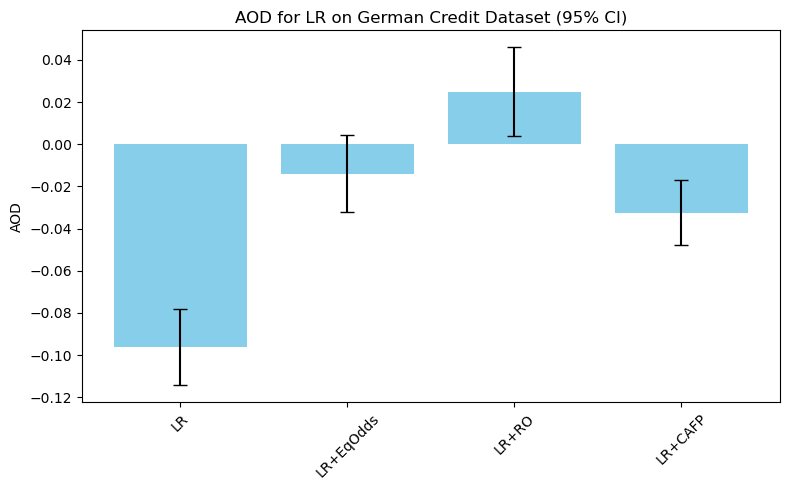}
            \caption{Logistic Regresion}
            \label{fig:LR - adult}
    \end{subfigure}%
      \hfill
    \begin{subfigure}[b]{0.5\textwidth}
            \centering
            \includegraphics[width=\textwidth]{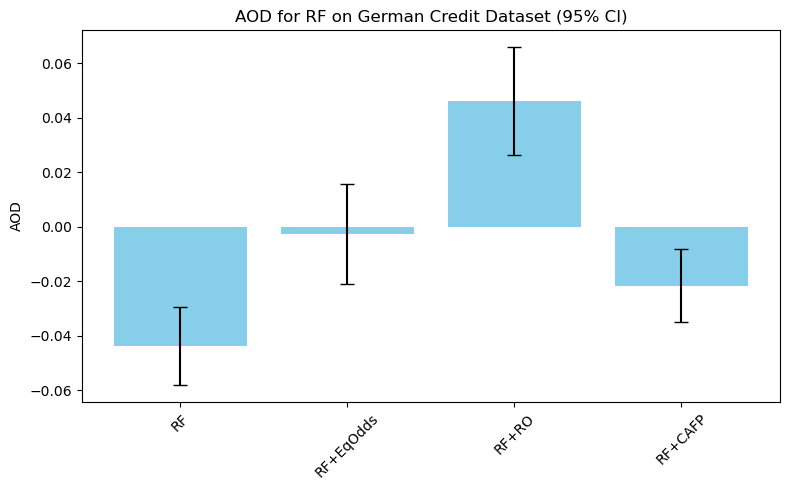}
            \caption{Random Forest}
            \label{fig:RF - adult}
    \end{subfigure}
  \hfill
    \begin{subfigure}[b]{0.5\textwidth}
            \centering
            \includegraphics[width=\textwidth]{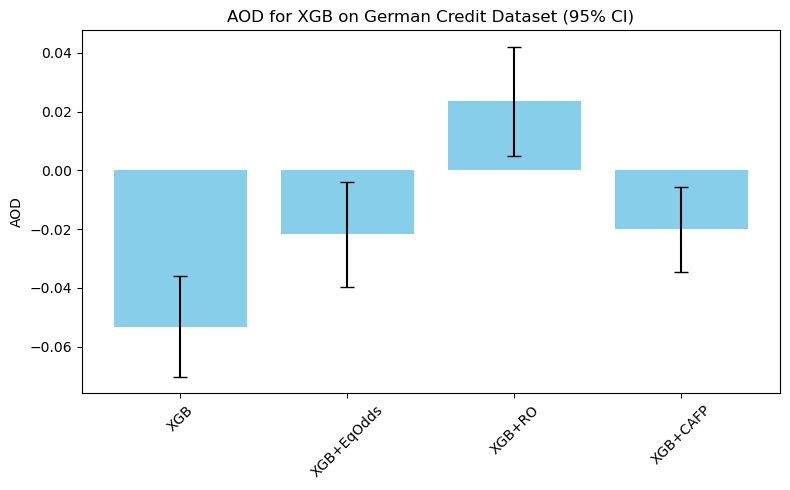}
            \caption{XGBoost}
            \label{fig:XGB - adult}
    \end{subfigure}
    \caption{AOD on German Credit dataset. Error bars indicate 95\% confidence intervals.}\label{fig:aod_german}
\end{figure*}

\begin{figure*}[!tbp]
\centering
    \begin{subfigure}[b]{0.5\textwidth}            
            \includegraphics[width=\textwidth]{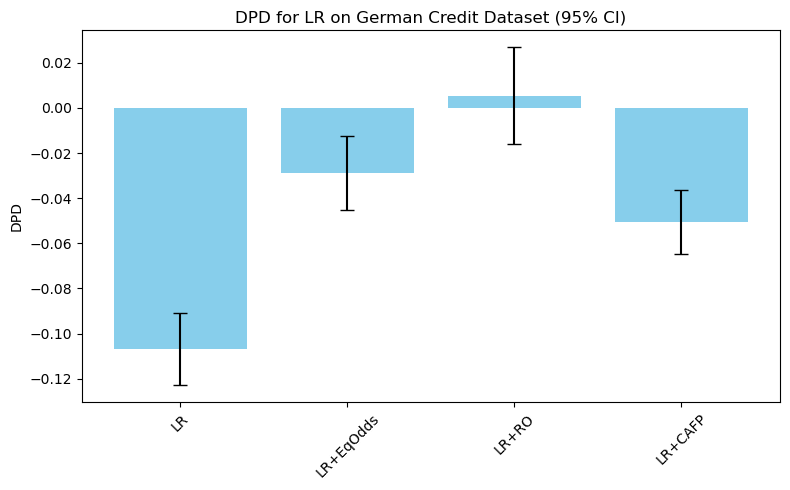}
            \caption{Logistic Regresion}
            \label{fig:LR - adult}
    \end{subfigure}%
      \hfill
    \begin{subfigure}[b]{0.5\textwidth}
            \centering
            \includegraphics[width=\textwidth]{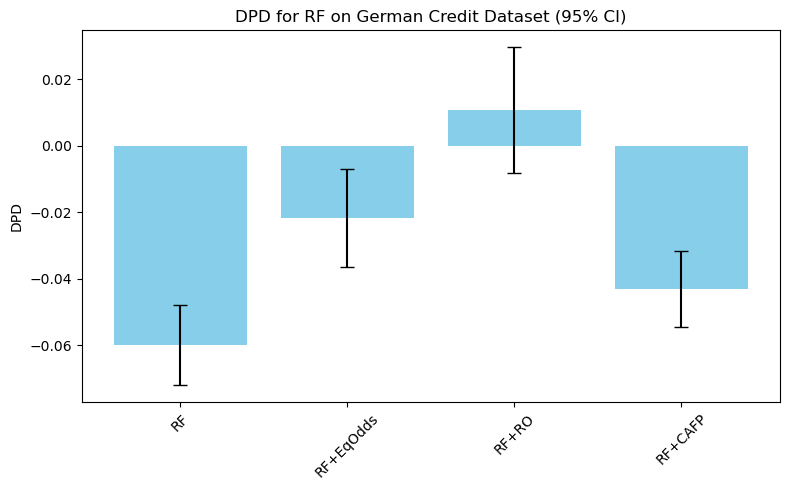}
            \caption{Random Forest}
            \label{fig:RF - adult}
    \end{subfigure}
  \hfill
    \begin{subfigure}[b]{0.5\textwidth}
            \centering
            \includegraphics[width=\textwidth]{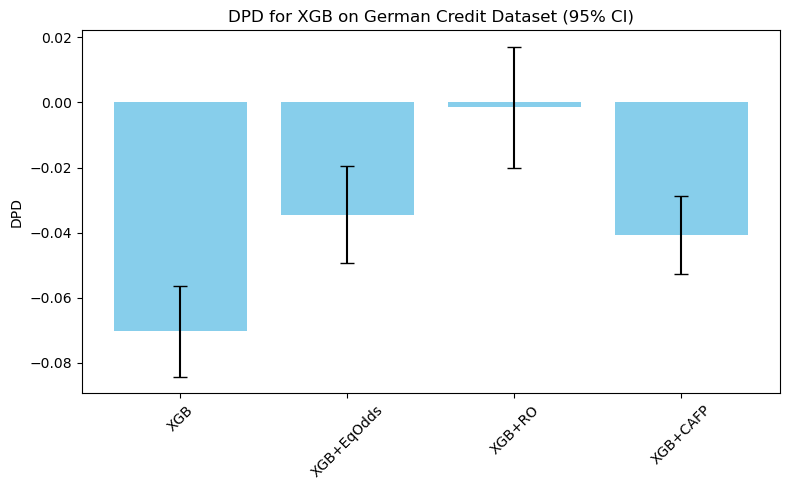}
            \caption{XGBoost}
            \label{fig:XGB - adult}
    \end{subfigure}
    \caption{DPD on German Credit dataset. Error bars indicate 95\% confidence intervals.}\label{fig:dpd_german}
\end{figure*}

\subsection{Threshold Sensitivity Analysis}

To assess the robustness of the proposed CAFP method, we conduct a threshold sensitivity analysis, examining how fairness and utility metrics behave under varying classification thresholds. While CAFP modifies the probabilistic output of the base classifier through counterfactual averaging, real-world deployment often requires thresholding these scores to produce binary decisions. Understanding how sensitive fairness outcomes are to this decision boundary is crucial for practical adoption.

To perform this experiment we vary the decision threshold \( t \in [0.01, 0.99] \) across 25 points and evaluate classification performance and fairness metrics at each setting. Specifically, we track accuracy (blue color) and demographic parity difference (DPD) (red color) across thresholds for the same three datasets and three models as before. We compare the original model before post-processing (solid line), the model after post-processing with Calibrated Equalized Odds (dotted line), and CAFP (dashed line). 

\begin{figure*}[!tbp]
\centering
    \begin{subfigure}[b]{0.5\textwidth}            
            \includegraphics[width=0.9\textwidth]{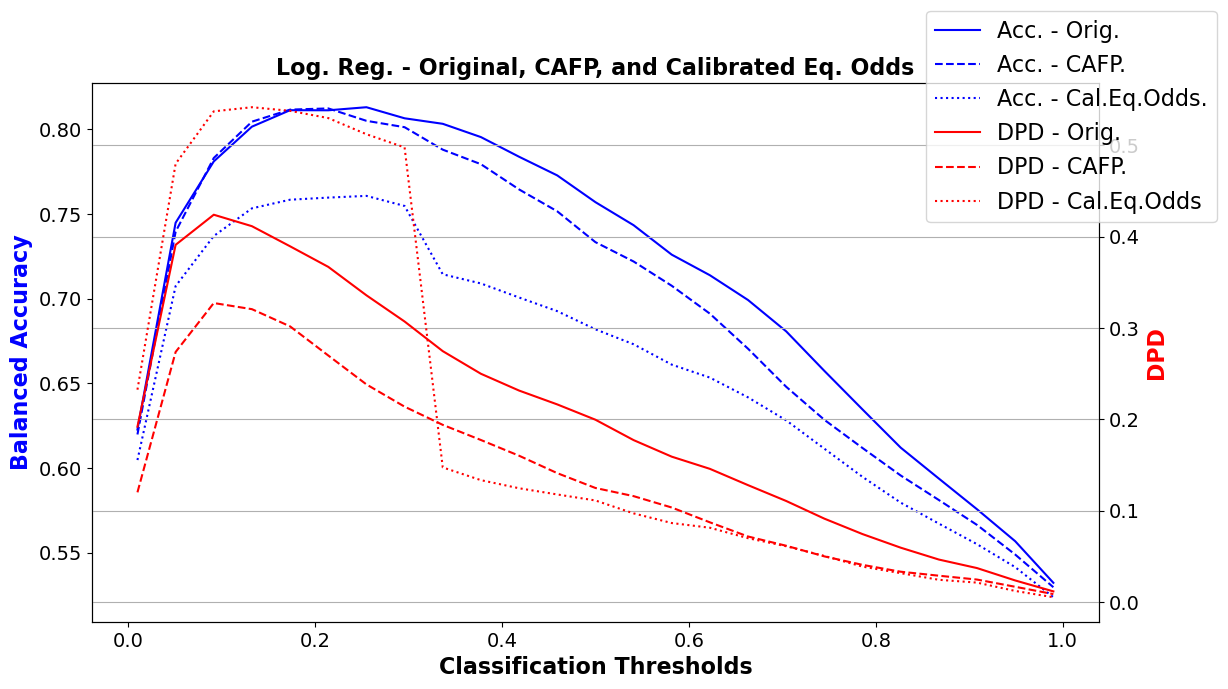}
            \caption{Logistic Regresion}
            \label{fig:LR - adult}
    \end{subfigure}%
      \hfill
    \begin{subfigure}[b]{0.5\textwidth}
            \centering
            \includegraphics[width=0.9\textwidth]{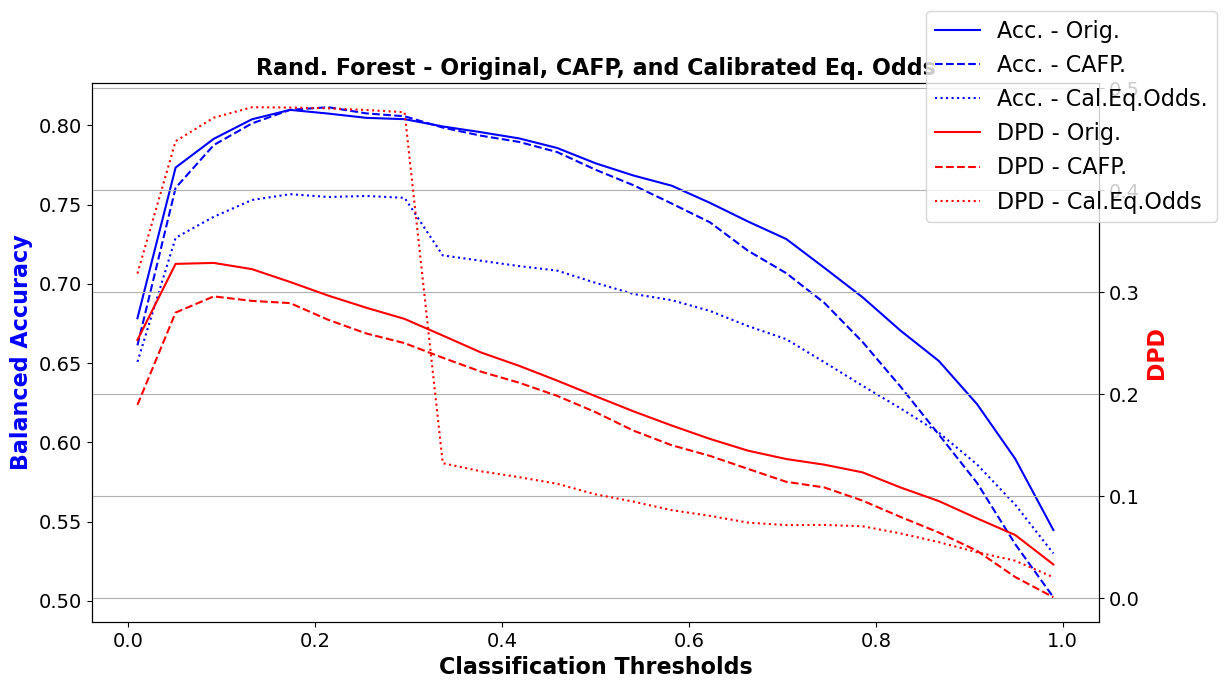}
            \caption{Random Forest}
            \label{fig:RF - adult}
    \end{subfigure}
  \hfill
    \begin{subfigure}[b]{0.5\textwidth}
            \centering
            \includegraphics[width=0.9\textwidth]{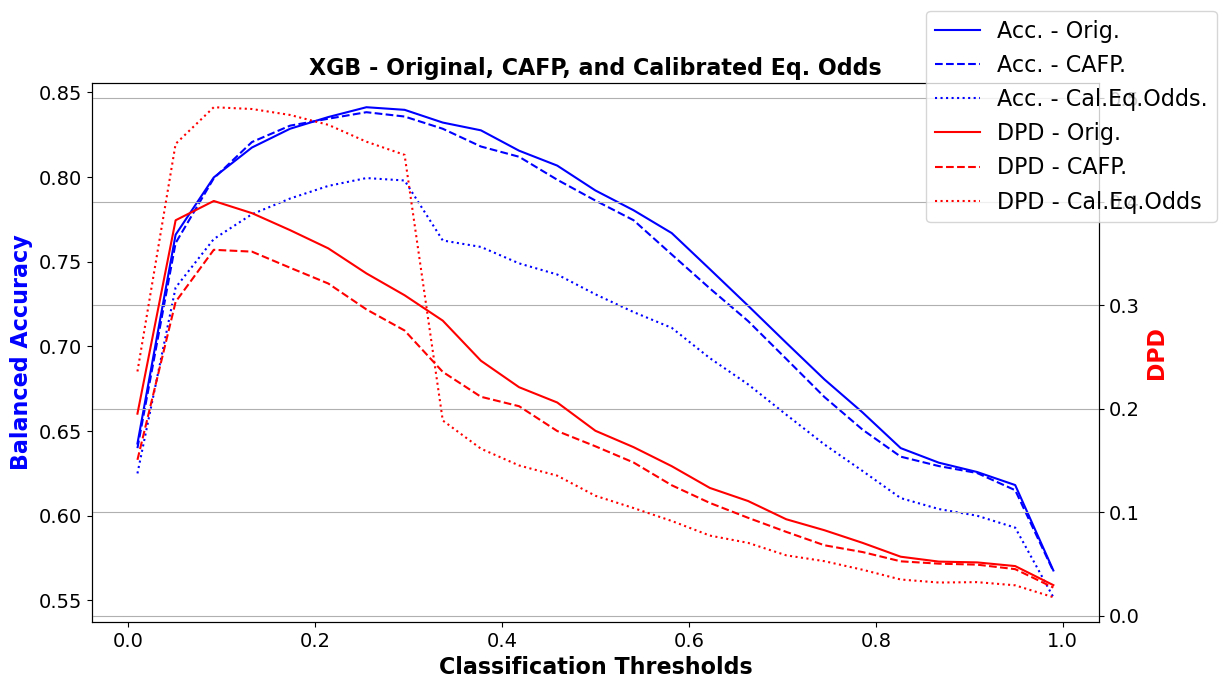}
            \caption{XGBoost}
            \label{fig:XGB - adult}
    \end{subfigure}
    \caption{Balanced Accuracy vs. DPD across thresholds - Adult dataset}\label{fig:TOF_adult}
\end{figure*}

\begin{figure*}[!tbp]
\centering
    \begin{subfigure}[b]{0.5\textwidth}            
            \includegraphics[width=0.9\textwidth]{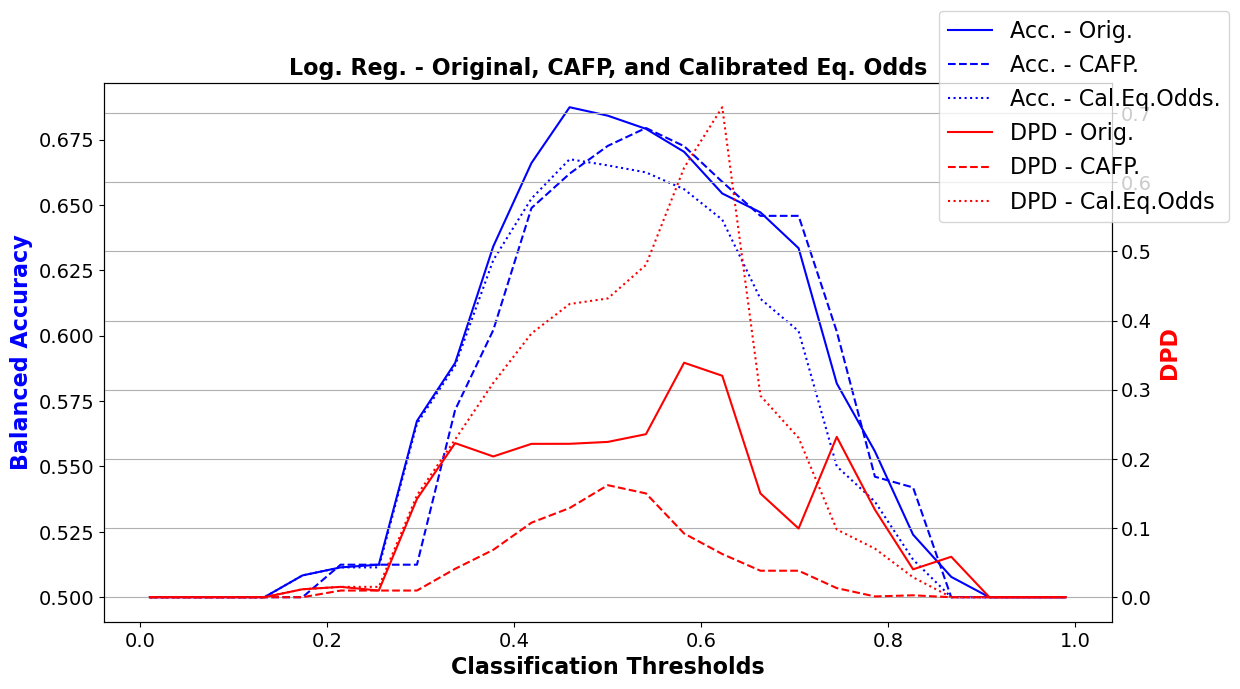}
            \caption{Logistic Regresion}
            \label{fig:LR - compas}
    \end{subfigure}%
      \hfill
    \begin{subfigure}[b]{0.5\textwidth}
            \centering
            \includegraphics[width=0.9\textwidth]{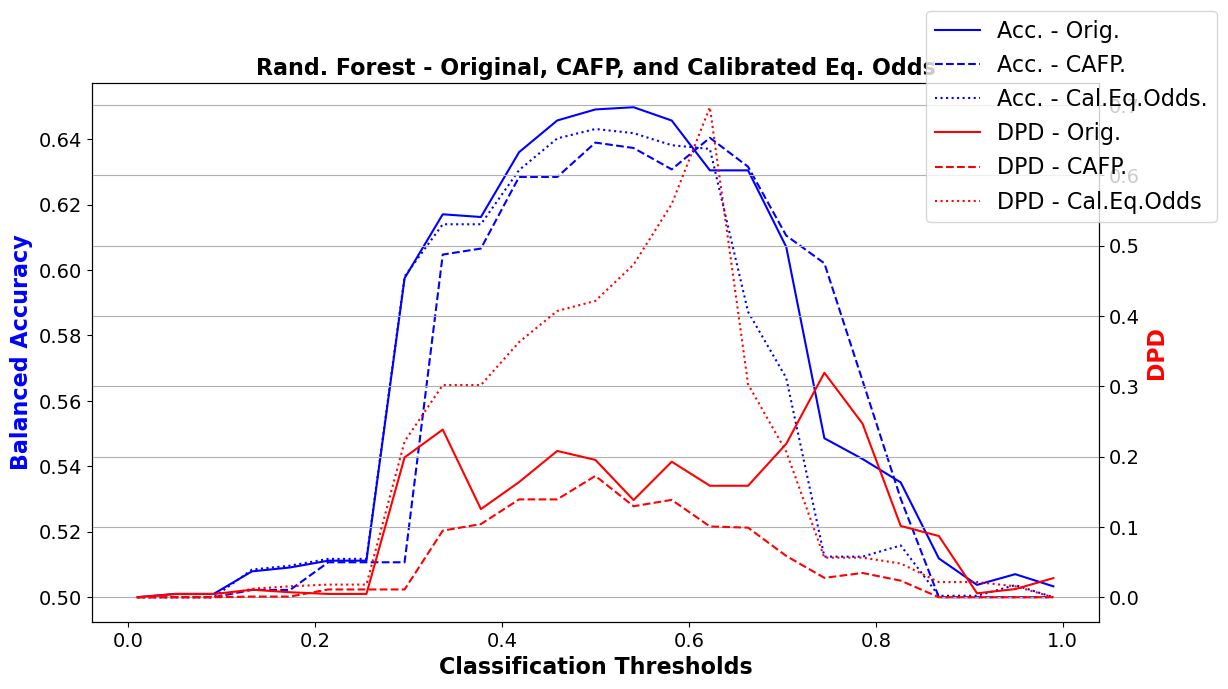}
            \caption{Random Forest}
            \label{fig:RF - compas}
    \end{subfigure}
  \hfill
    \begin{subfigure}[b]{0.5\textwidth}
            \centering
            \includegraphics[width=0.9\textwidth]{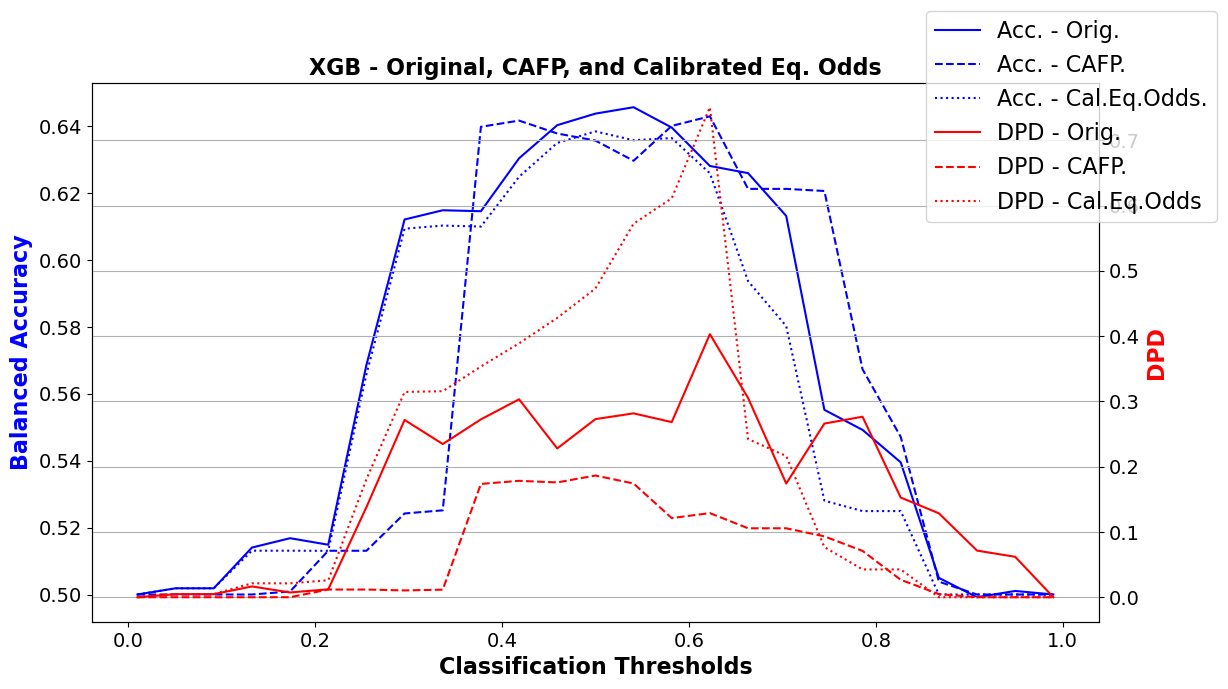}
            \caption{XGBoost}
            \label{fig:XGB - compas}
    \end{subfigure}
    \caption{Balanced Accuracy vs. DPD across thresholds - COMPAS dataset}\label{fig:TOF_compas}
\end{figure*}

\begin{figure*}[!tbp]
\centering
    \begin{subfigure}[b]{0.5\textwidth}            
            \includegraphics[width=0.9\textwidth]{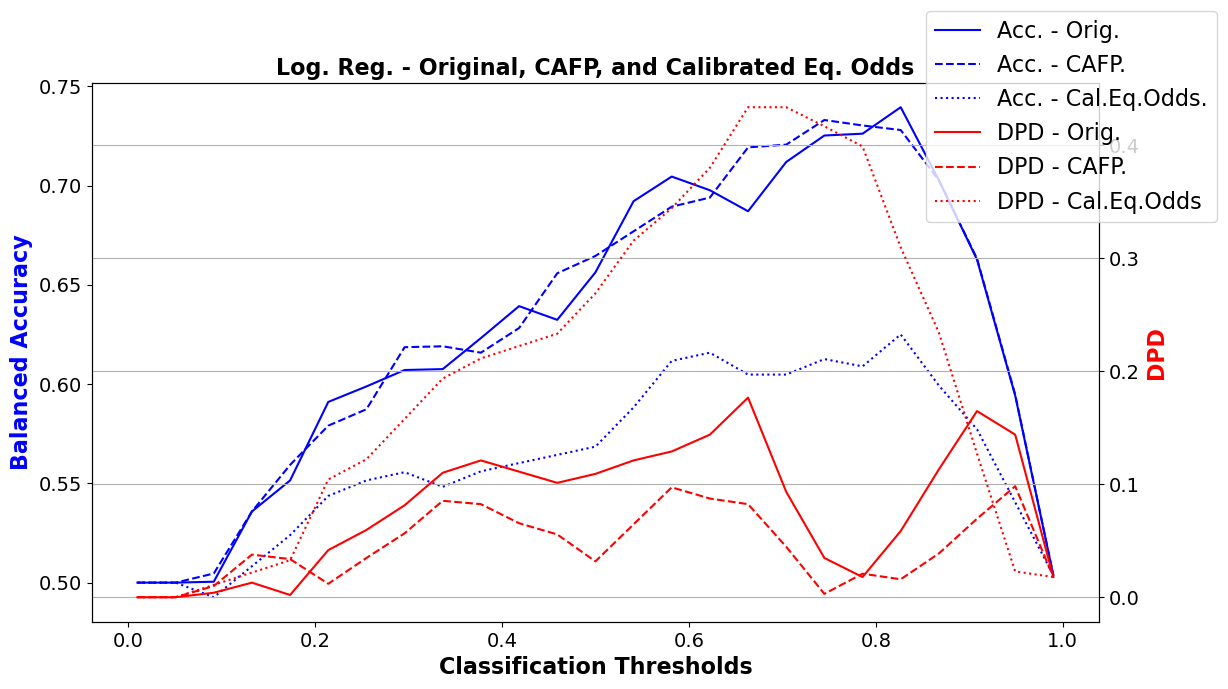}
            \caption{Logistic Regresion}
            \label{fig:LR - german}
    \end{subfigure}%
      \hfill
    \begin{subfigure}[b]{0.5\textwidth}
            \centering
            \includegraphics[width=0.9\textwidth]{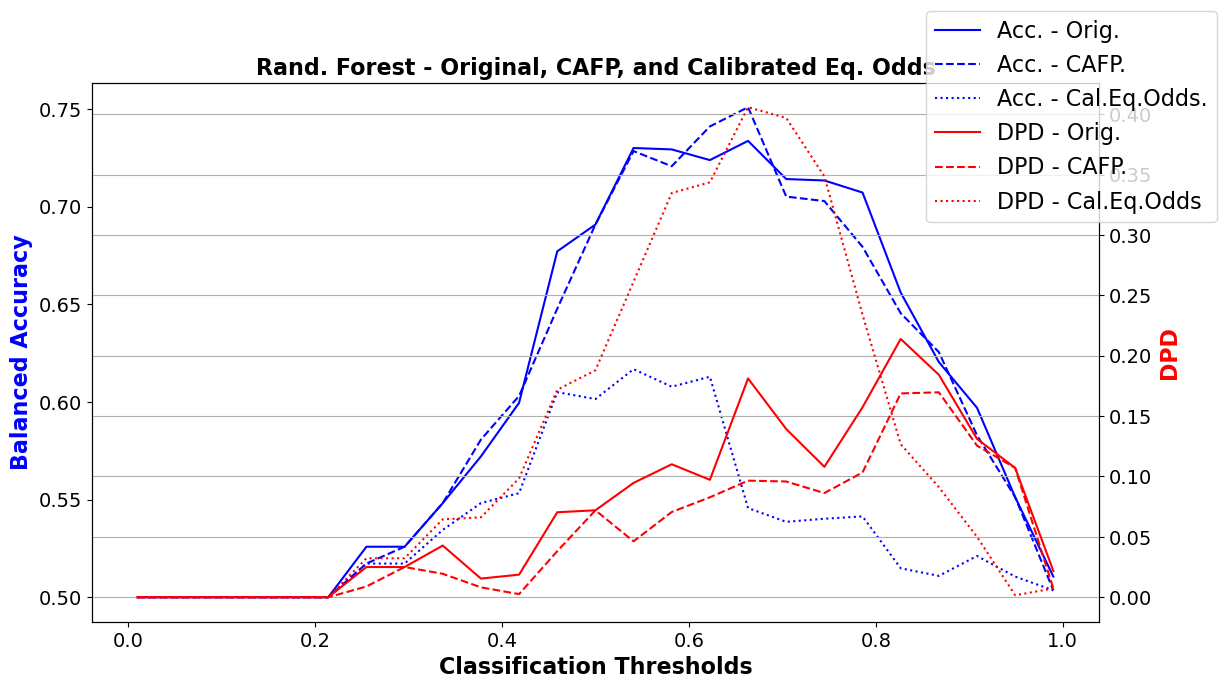}
            \caption{Random Forest}
            \label{fig:RF - german}
    \end{subfigure}
  \hfill
    \begin{subfigure}[b]{0.5\textwidth}
            \centering
            \includegraphics[width=0.9\textwidth]{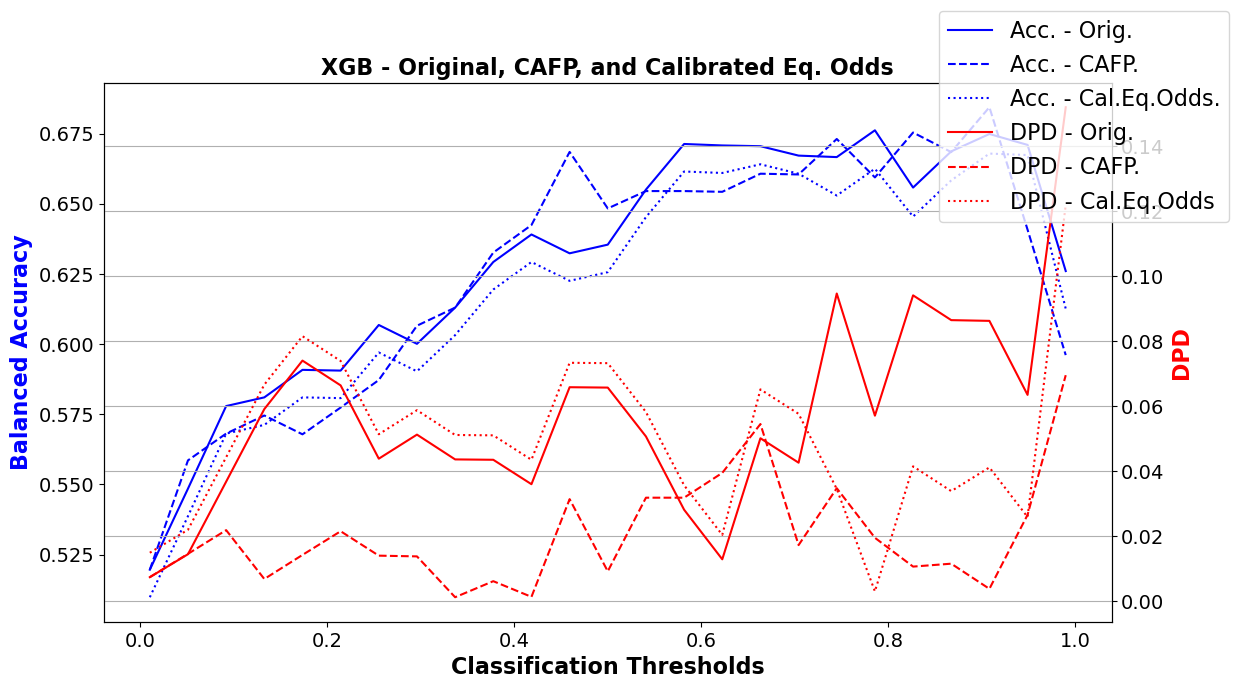}
            \caption{XGBoost}
            \label{fig:XGB - german}
    \end{subfigure}
    \caption{Balanced Accuracy vs. DPD across thresholds - German Credit dataset}\label{fig:TOF_german}
\end{figure*}

Figures~\ref{fig:TOF_adult}, \ref{fig:TOF_compas}, and \ref{fig:TOF_german} summarize the threshold-dependent trade-offs between balanced accuracy and demographic parity difference (DPD) across three datasets—Adult Income, COMPAS, and German Credit—and three classifiers: logistic regression, random forest, and XGBoost. For each configuration, we compare the original model (solid lines), calibrated equalized odds (dotted lines), and our proposed CAFP method (dashed lines). Across nearly all settings, CAFP maintains a favorable balance: it preserves high balanced accuracy comparable to the base model, while substantially reducing group disparity across thresholds. Calibrated Equalized Odds achieves lower DPD at some thresholds but often suffers from reduced accuracy and less stability. Notably, in the Adult dataset (Figures~\ref{fig:LR - adult}, \ref{fig:RF - adult}, and \ref{fig:XGB - adult}), CAFP curves closely track the original model’s accuracy while lowering DPD across the entire threshold range. For the COMPAS dataset (Figures~\ref{fig:LR - compas}, \ref{fig:RF - compas}, \ref{fig:XGB - compas}), where the original models exhibit high disparity, CAFP consistently mitigates unfairness with minimal utility loss. In the German Credit setting (Figures~\ref{fig:LR - german}, \ref{fig:RF - german}, \ref{fig:XGB - german}), where base models already show moderate fairness, CAFP performs conservatively—preserving fairness gains without degrading performance. Overall, these results confirm that CAFP offers smooth and stable fairness–accuracy behavior across threshold choices, architectures, and datasets, making it a robust post-processing tool suitable for real-world deployment scenarios where decision thresholds may vary or require tuning.

\subsection{Ablation Study: Factual, Counterfactual, and Averaged Predictions}

To isolate the contribution of counterfactual averaging in CAFP, we conducted an ablation study comparing three variants of prediction strategies:
\begin{itemize}
    \item \textbf{Factual only:} \( f(x, a) \), the model’s output using the true protected attribute.
    \item \textbf{Counterfactual only:} \( f(x, 1 - a) \), the model’s output with the protected attribute flipped.
    \item \textbf{CAFP:} \( \hat{f}(x) = \frac{1}{2}(f(x, a) + f(x, 1 - a)) \), the counterfactual average.
\end{itemize}

The table below shows test set results on the Adult dataset using logistic regression. While the factual and counterfactual models yield comparable accuracy and fairness disparities, the averaged CAFP method significantly reduces both Demographic Parity Difference (DPD) and Average Odds Difference (AOD) without sacrificing performance.

\begin{table}[ht]
\centering
\caption{Ablation study comparing factual-only, counterfactual-only, and averaged predictions (Adult dataset, logistic regression).}
\label{tab:ablation}
\begin{tabular}{lccc}
\toprule
\textbf{Prediction Type} & \textbf{Accuracy} & \textbf{DPD} & \textbf{AOD} \\
\midrule
Factual \( f(x, a) \) & 0.8464 & -0.1868 & -0.0961 \\
Counterfactual & & \\
\( f(x, 1 - a) \) & 0.8463 & -0.1775 & -0.0912 \\
CAFP \( \hat{f}(x) \) & 0.8432 & \textbf{-0.1157} & \textbf{0.0075} \\
\bottomrule
\end{tabular}
\end{table}

These results confirm that the averaging mechanism in CAFP plays a crucial role in reducing group-level disparities. The method achieves this without introducing significant predictive degradation, reinforcing its suitability as a lightweight fairness-enhancing post-processing technique.

Similar patterns were observed with other datasets and models.

\subsection{Summary}

The experiments validate that CAFP is an effective post-processing strategy for reducing group disparities in binary classification. The method achieves consistent improvements in fairness metrics with minimal degradation in accuracy, and requires no retraining or access to model internals. We next discuss the practical implications and limitations of the approach.

\section{Discussion}

The results presented in Section~6 demonstrate that Counterfactual Averaging for Fair Predictions (CAFP) is a simple yet powerful post-processing technique for improving group fairness in binary classification tasks. In this section, we discuss the implications of our findings, contrast CAFP with existing approaches, and outline its practical considerations and limitations.

\subsection{Fairness-Utility Trade-off}

The experimental results across all datasets and model types demonstrate that Counterfactual Averaging for Fair Predictions (CAFP) consistently improves fairness metrics—specifically, demographic parity difference (DPD), and average odds difference (AOD)—while preserving predictive utility. In most configurations, CAFP achieves parity reductions comparable to or better than state-of-the-art post-processing methods, such as Equalized Odds and Reject Option Classification, but with significantly smaller degradation in accuracy.

The fairness–accuracy trade-off curves across thresholds further support this finding. CAFP consistently reduces group disparities across a wide threshold range, often tracking the baseline model’s accuracy curve while lowering disparity more effectively than calibrated equalized odds. In datasets with strong baseline bias (e.g., COMPAS), CAFP mitigates disparities with only marginal loss in accuracy, and in more balanced settings (e.g., German Credit), it preserves fairness without unnecessary distortion. These results align with the theoretical guarantees presented in Section~\ref{sec:theory} and confirm that CAFP offers a robust, model-agnostic trade-off between group fairness and predictive performance.

%Across all evaluated datasets, CAFP significantly reduces group disparities—measured by demographic parity difference (DPD), equalized odds difference (EOD), and counterfactual bias—while maintaining overall accuracy within a fraction of a percentage point. The fairness–accuracy trade-off curves confirm that CAFP offers favorable operating points where fairness improves substantially with negligible impact on performance. This supports the theoretical bounds on prediction distortion and highlights the utility of counterfactual averaging as a fairness-preserving mechanism.

\subsection{Comparison with Existing Methods}

Unlike many in-processing or threshold-based post-processing techniques, CAFP is model-agnostic, requires no retraining, and does not involve complex optimization procedures. Its primary advantage over group-specific thresholding methods is that it does not require the protected attribute to be known or observed at test time, as predictions for both counterfactual group values are computed and averaged internally. Compared to fairness methods based on adversarial learning or causal graphs, CAFP is more lightweight and practical for deployment in real-world systems where full model transparency or causal structure is unavailable.

\subsection{Generalizability and Applications}

CAFP is particularly well-suited for systems that expose a probabilistic API (e.g., scoring services in hiring or lending platforms), or where post-deployment fairness audits are necessary. Because it operates independently of the training pipeline, it can be added to legacy models or third-party black-box systems. Furthermore, since CAFP treats the protected attribute symmetrically, it can be applied in contexts where the definition of "advantaged" vs. "disadvantaged" group is not explicitly encoded in the model.

As a real-world example of use, consider a fintech company that uses a proprietary, third-party credit scoring model to predict loan default risk. The model is deployed as a black-box API and includes protected attributes such as gender or age among its inputs. Due to regulatory obligations (e.g., the Equal Credit Opportunity Act), the company must ensure that its lending decisions are not unfairly biased by these sensitive attributes. However, since the model is externally maintained, the firm cannot retrain or alter its internal parameters.

To address this, the company applies Counterfactual Averaging for Fair Predictions (CAFP) as a post-processing fairness layer. For each applicant, the firm queries the model twice—once using the actual value of the protected attribute (e.g., gender = female) and once with the counterfactual (e.g., gender = male)—and averages the two resulting probability scores. The final score is used to make the loan approval decision. This simple wrapper allows the company to reduce disparate treatment across demographic groups, without modifying the original model or relying on protected attribute information at test time. As a result, CAFP enables compliance with fairness regulations while preserving predictive accuracy and operational simplicity, illustrating its value in real-world financial decision-making contexts.

\subsection{Limitations}

Despite its strengths, CAFP has certain limitations. First, it assumes the protected attribute is binary; extending it to multi-class or intersectional identities (e.g., gender × race) is non-trivial and left for future work. Second, while CAFP reduces the influence of the protected attribute, it does not eliminate all indirect effects—especially when feature distributions differ substantially across groups. Additionally, the method assumes the ability to simulate predictions for alternate group memberships, which may not always be possible if the protected attribute cannot be synthetically perturbed.

\subsection{Toward Broader Fairness Guarantees}

Although our theoretical analysis focuses on demographic parity and equalized odds, CAFP could be integrated into broader fairness frameworks. For example, it may be combined with calibration or ranking constraints to ensure multi-objective fairness compliance. It also opens the door to designing inference-time fairness wrappers that can plug into diverse pipelines without disrupting the underlying learning system.

\subsection{Ethical and Societal Implications}

As machine learning systems increasingly inform decisions in sensitive domains such as credit lending, hiring, and criminal justice, ensuring equitable treatment across demographic groups becomes a pressing ethical obligation. The proposed CAFP method contributes to this objective by offering a practical, model-agnostic mechanism for reducing algorithmic bias post hoc, without requiring retraining or access to internal model parameters.

Importantly, CAFP addresses key deployment constraints: it does not rely on access to sensitive attributes at test time, avoids group-specific thresholds that may raise legal concerns, and minimizes prediction distortion, thereby preserving individual-level consistency. These properties make CAFP especially suitable for real-world applications in privacy-sensitive or regulated settings, where fairness interventions must be both auditable and minimally intrusive.

However, we recognize that technical solutions alone cannot resolve the full spectrum of ethical challenges in AI systems. CAFP mitigates direct influence of protected attributes but does not eliminate structural or historical biases embedded in training data or societal institutions. Moreover, while counterfactual averaging reduces group disparities, it may obscure individual-level harms or mask intersectional inequities if subgroup differences are not explicitly modeled.

We encourage practitioners to view CAFP as part of a broader fairness toolkit, to be deployed alongside transparency audits, stakeholder consultation, and ongoing monitoring. Future extensions should explore accountability in multi-attribute or dynamic settings, and assess how users interpret and trust predictions under counterfactual adjustment. Ultimately, the ethical deployment of fair machine learning systems requires aligning technical rigor with societal values, legal frameworks, and inclusive governance.

\section{Conclusion and Future Work}

This paper introduced Counterfactual Averaging for Fair Predictions (CAFP), a model-agnostic post-processing method for mitigating unfair dependence on protected attributes in binary classification. Inspired by the principle of counterfactual fairness, CAFP operates by averaging a model's predictions over factual and counterfactual versions of the protected attribute. The method is simple to implement, computationally efficient, and applicable to any black-box classifier that accepts group attributes as input.

We presented a formal theoretical analysis demonstrating that CAFP removes direct dependence on the protected attribute under mild assumptions, introduces provably bounded prediction distortion, and improves common group fairness metrics such as demographic parity and equalized odds. Our empirical evaluation across three widely studied datasets—Adult Income, COMPAS, and German Credit—confirmed that CAFP significantly reduces group-level disparities while preserving predictive accuracy. These results validate CAFP as a practical and principled fairness mechanism for real-world deployment.

Looking ahead, several directions remain open for future work. First, we aim to extend CAFP to handle non-binary or multi-valued protected attributes, including intersectional subgroups (e.g., race × gender). Second, we plan to explore probabilistic and distributional variants of counterfactual averaging that explicitly account for uncertainty in group assignments or latent causal pathways. Finally, we are interested in integrating CAFP into fairness-aware calibration and ranking frameworks, and in evaluating its impact on downstream tasks such as allocation and recommendation.

Overall, CAFP provides a promising building block for fairness-aware machine learning, particularly in settings where model internals are inaccessible and post hoc interventions are needed.

\bibliography{sn-bibliography}% common bib file

\end{document}